\pdfoutput=1
\documentclass{article}
\label{key}
\usepackage{hyperref}
\usepackage{authblk}
\usepackage{amsmath}
\usepackage{amssymb}
\usepackage{amsthm}
\usepackage{comment}
\usepackage{booktabs}
\usepackage{xcolor}
\newtheorem{problem}{Problem}
 
\usepackage{tikz,ifthen,pgfplots}                       
\usetikzlibrary{shapes.geometric}
\usetikzlibrary{decorations}
\usetikzlibrary{arrows,decorations.pathmorphing,backgrounds,positioning,fit,petri}
\usetikzlibrary{decorations.pathreplacing} 
\usetikzlibrary{decorations.markings} 
\usetikzlibrary{decorations.shapes} 
\usetikzlibrary{arrows.meta} 
\usetikzlibrary{quotes,angles} 
\usetikzlibrary{positioning}
\usepgfplotslibrary{fillbetween}
\usetikzlibrary{patterns} 
\usetikzlibrary{patterns,chains}
\usetikzlibrary{hobby}
\usepackage{tikz-3dplot}

\usepackage{pgfplots}

\usepackage{graphics}

\usepackage{enumitem}

\usepackage[caption=false]{subfig}

\usepackage[ruled]{algorithm2e}
\usepackage{fancybox}

\newcommand{\cl}{\mathcal{C}}

\newcommand{\pot}{2^E}

\newcommand{\opn}[1]{\operatorname{#1}}

\newcommand{\conv}[1]{\operatorname{conv}(#1)}

\newcommand{\covl}{\operatorname{C}_\downarrow}

\newcommand{\covu}{\operatorname{C}_\uparrow}

\newcommand{\NP}{$\opn{NP}$}

\newcommand{\HSP}{\textsc{Half-Space Separation}}

\newcommand{\MCSep}{MCSS}

\newcommand{\bigO}[1]{O \left( #1 \right)}

\newcommand{\AlgoSep}{MCSS}

\newcommand{\clop}{\rho}

\newcommand{\clSystem}{(E, \cl)}

\newcommand{\clSystemcc}{(E, \cl_\cc)}

\newcommand{\Rd}{\mathbb{R}^d}

\newcommand{\cg}{\gamma}

\newcommand{\cc}{\rho}

\newcommand{\cla}{\lambda}

\newcommand{\convexhull}{\alpha}

\newcommand{\vclog}{(V,\mathcal{C}_{\cg})}

\newcommand{\cC}{\mathcal{C}}

\newcommand{\Top}[1]{\sup #1}
\newcommand{\Bot}[1]{\inf #1}
\newcommand{\PosTop}{\Top{A}}
\newcommand{\PosBot}{\Bot{A}}
\newcommand{\NegTop}{\Top{B}}
\newcommand{\NegBot}{\Bot{B}}
\newcommand{\TopA}{\Top{A}}
\newcommand{\BotA}{\Bot{A}}
\newcommand{\TopB}{\Top{B}}
\newcommand{\BotB}{\Bot{B}}
\newcommand{\TopI}{\Top{I}}
\newcommand{\BotF}{\Bot{F}}
\newcommand{\joinp}[1]{\sup\{#1\}}
\newcommand{\meetp}[1]{\inf\{#1\}}
\newcommand{\join}[1]{\sup}

\newcommand{\meet}[1]{\inf}

\newcommand{\topL}{\top_L}
\newcommand{\botL}{\bot_L}

\newcommand{\gen}{\operatorname{gen}}

\newcommand{\No}{{\sc No}}

\newtheorem{Proposition}{Proposition}
\newtheorem{Theorem}{Theorem}
\newtheorem{Corollary}{Corollary}
\newtheorem{Example}{Example}
\newtheorem{remark}{Remark}
\newtheorem{lemma}{Lemma}

\title{Maximal Closed Set and Half-Space Separations in  Finite Closure Systems\thanks{An early version of this paper appeared in \cite{Seiffarth19}.}
}

\author[1]{Florian Seiffarth}
\author[1,2,3]{Tam\'{a}s Horv\'{a}th}
\author[1,2,3]{Stefan Wrobel}

\affil[ ]{\footnotesize\textit {\{seiffarth, horvath, wrobel\}@cs.uni-bonn.de}}
\affil[1]{\footnotesize Dept. of Computer Science, University of Bonn, Bonn,
	Germany}
\affil[2]{\footnotesize Fraunhofer IAIS, 
	Schlo{\ss} Birlinghoven, Sankt Augustin, Germany}
\affil[3]{\footnotesize Fraunhofer Center for Machine Learning, 
	Schlo{\ss} Birlinghoven, Sankt Augustin, Germany}

\begin{document}
	
%
	
\date{ }

	\maketitle

	\begin{abstract}
		Several concept learning problems can be regarded as special cases of half-space separation in abstract closure systems over finite ground sets. 
		For the typical scenario that the closure system is implicitly given via a closure operator, we show that the half-space separation problem is NP-complete.
		As a first approach to overcome this negative result, we relax the problem to maximal closed set separation, give a generic greedy algorithm solving this problem with a linear number of closure operator calls, and show that this bound is sharp.
		For a second direction, we consider  Kakutani closure systems and prove that they are algorithmically characterized by the greedy algorithm.
		As a first special case of the general problem setting, we consider Kakutani closure systems over graphs and give a sufficient condition for this kind of closure systems in terms of forbidden graph minors.
		For a second special case, we then focus on closure systems over finite lattices, give an improved adaptation of the generic greedy algorithm, and present an application concerning subsumption lattices. 
	\end{abstract}

	
	\section{Introduction}
	We consider the following computational problem: Given a \textit{closure system}~\cite{Davey/Priestley/2002} $\cl \subseteq 2^E$ over some \textit{finite} ground set $E$ and $A,B \subseteq E$, find $A',B' \in \cl$ such that $A \subseteq A'$, $B \subseteq B'$, and $A'$ and $B'$ form a partitioning of $E$, if such $A'$ and $B'$ exist; otherwise return the answer ``\No".
	This problem can be regarded as an adaptation of that of binary \textit{half-space separation} in $\Rd$, a well-studied problem in \textit{machine learning} (see, e.g., \cite{Rosenblatt58,BoserGuyonVapnik92,Vapnik/98,FreundS99}). 
	Indeed, this latter problem can be viewed as follows: The underlying closure system is given by the family $\cl'$ of all (topologically) closed convex subsets of $\Rd$ and for all disjoint sets $X,Y \subseteq \Rd$, the task is to return two half-spaces $X',Y' \in \cl'$ such that $X \subseteq X'$, $Y \subseteq Y'$, and $X'$ is the complement of $Y' \setminus H$, where $H$ is the boundary hyperplane of $Y'$, if the convex hulls of $X$ and $Y$ are disjoint. Otherwise, the algorithm is required to return the answer ``\No".
	The correctness of this generic method for $\Rd$ follows from the result of \cite{Kakutani37} that any two disjoint convex sets in $\Rd$ are always separable by a hyperplane.
	In order to reflect this problem adaptation, the output closed sets $A',B'$ of the generic problem above are referred to as \textit{(abstract) half-spaces} and the input sets $A,B$ as \textit{half-space separable} sets in $\cl$. 
	
	While half-space separability in $\Rd$ is a well-understood problem, its adaptation to closure systems over finite ground sets has received less attention in the machine learning literature. This is somewhat surprising, as several problems including, for example, node classification in networks (see, e.g., \cite{ThiessenG21}) or certain consistent hypothesis finding problems in inductive logic programming~(see, e.g., \cite{NieWol97}) can be viewed as special cases of half-space separation in finite closure systems. 	
	In contrast, several results concerning different formal properties of \textit{abstract} half-spaces over finite domains have been published in \textit{theoretical computer science} and \textit{geometry}~(see, e.g., \cite{Chepoi1994,ellis1952,Kubis2002,vandeVel1984}).
	Using the fact that the family of convex sets in $\Rd$ forms a {\em closure system}, the underlying idea of adapting hyperplane separation in $\Rd$ to  arbitrary finite ground sets $E$ is to consider some semantically meaningful closure system $\cl$ over $E$. A subset $H$ of $E$ is considered as an {\em (abstract) half-space}, if $H$ and its complement both belong to $\cl$. In this field of research there is a distinguished focus on characterization results of {\em Kakutani closure systems} (see, e.g.,~\cite{Chepoi1994,Vel1993}). This kind of closure systems satisfy the property that any two subsets of the ground set are half-space separable in the closure system if and only if their closures are disjoint.     
	
	Utilizing the results of other research fields~\cite{Chepoi1994,ellis1952,Kubis2002,vandeVel1984}, in this work we first study the {\em algorithmic} aspects of half-space separation in finite closure systems.
	In the problem setting we assume that the closure systems are given \textit{implicitly} by a {\em closure operator}. This assumption is justified by the fact that their cardinality can be exponential in that of the underlying ground set. 
	In addition, we assume that no additional domain specific information is available for the algorithm.
	We show that half-space separability in this generic problem setting is \NP-hard. 
	We note that this negative result is independent of the time complexity of the closure operator, i.e., it holds also for the case that the closure of a set can be calculated in unit time.
	If the closure operator can be computed in time polynomial in the size of the input set, the underlying half-space separation problem lies in {\NP} (i.e., half-space separation is \NP-complete). 

	To overcome this negative complexity result, we first relax the problem to {\em maximal} closed set\footnote{Throughout this work we consistently use the nomenclature ``closed sets'' by noting that ``convex'' and ``closed'' are synonyms by the standard terminology of this field.} separation.
	That is, we are interested in finding two closed sets that are disjoint, contain the two input sets, and none of them has a proper superset in the closure system possessing these properties. 
	For this relaxed problem we give a simple \textit{greedy} algorithm and show that it is \textit{efficient} and \textit{optimal} in terms of the number of closure operator calls. 
	As a second approach, we then focus on \textit{Kakutani} closure systems.
	We first show that in order to decide whether a  closure system is Kakutani, any 
	algorithm that has access to the closure system only via the corresponding closure operator requires \textit{exponentially} many closure operator calls in the worst-case. 
	Still, Kakutani closure systems remain highly interesting because there are various closure systems that are known to be Kakutani cf.~\cite{Chepoi1994}.
	We also prove that the above mentioned simple greedy algorithm provides an \textit{algorithmic} characterization of Kakutani closure systems. That is, for all $A, B \subseteq E$, the output $A' \supseteq A$ and $B' \supseteq B$ of the greedy algorithm is a partitioning of $E$ if and only if the closures of $A$ and $B$ are disjoint.
	
	In the second part of the paper we consider \textit{domain specific} adaptations of maximal closed set and half-space separations to closure systems over \textit{graphs} and \textit{lattices}. 
	In particular, we show that closure systems over graphs induced by shortest paths~~\cite{doi:10.1137/0607049} are Kakutani if they do not contain the bipartite clique $K_{2,3}$ as a minor. The converse of this claim is, however, not true.
	This result, together with the characterization result of \cite{Chartrand/Harary/1967} immediately implies that closure systems over \textit{outerplanar} graphs and hence, over trees are Kakutani.
%
%
	Regarding \textit{lattices}, we present an adaptation of the above mentioned greedy algorithm to closure systems over lattices. It computes a disjoint maximal ideal and filter that contain the two input sets.
	This adaptation has several algorithmic advantages over the generic greedy algorithm. 
	In particular,  it utilizes the facts that  in each iteration, the current \textit{ideal} and \textit{filter} can be represented by its \textit{supremum} and \textit{infimum}, respectively. Furthermore, their disjointness can be decided by comparing these two elements.
	For the special case that the elements of the lattice are subsets of some finite ground set, the number of closure operator calls is only quadratic in the cardinality of the domain, implying an exponential speed-up over the generic greedy algorithm mentioned above. 
	In addition to these results, we also show that the adaptation of the greedy algorithm preserves the characterization property, i.e., it provides an \textit{algorithmic} characterization of Kakutani closure systems over finite lattices. This characterization result is somewhat orthogonal to the characterization given in terms of \textit{distributivity} (see, e.g.,~\cite{Kubis2002}).

	The rest of the paper is organized as follows. In Section~\ref{sec:relatedwork} we discuss related work.
	In Section~\ref{sec:preliminaries}	we collect the necessary notions and fix the notation. 
	In Section~\ref{sec:maximalclosed} we define the problem settings, discuss potential applications, and study the complexity issues of half-space and maximal closed set separation in abstract closure systems over finite domains.
	Section~\ref{sec:kakutani} is devoted to Kakutani closure systems. 
	Section~\ref{sec:domainspecific} is concerned with closure systems over graphs and lattices. 
	Finally, in Section~\ref{sec:summary} we formulate some problems for further research.

\section{Related Work}\label{sec:relatedwork}
Similarly to hyperplane separations in Euclidean spaces, different kinds of separations can be considered for (finite) closure systems using the definition of closed sets. In \cite{vandeVel1984} the most familiar separation axioms $(S1)$ - $(S4)$ are presented. They can be described as follows: $(S1)$ all singletons are closed, $(S2)$ two distinct elements can be separated by half-spaces, $(S3)$ every closed set is an intersection of half-spaces and $(S4)$ two disjoint closed sets can be separated by half-spaces.\footnote{Note that if $(S1)$ holds true then the implications $(S4)\implies(S3)\implies(S2)$ follow. In general we do not assume that singletons are closed but e.g. in case of geodesic closures in graphs and closures in lattices $(S1)$ holds true. An example where $(S1)$ does not hold is the closure system induced by the galois closure operator used in closed itemset mining.} In this work we concentrate on closure systems satisfying the most restrictive axiom $S4$ inspired by the Euclidean space counterpart result of \cite{Kakutani37} that convex sets can be separated by hyperplanes. This is motivated by the fact that basic machine learning algorithms such as the Perceptron \cite{Rosenblatt58} and Support Vector Machines \cite{Vapnik/98} heavily rely on the separation property in $\Rd$. Hence we want to analyze Kakutani closure systems, i.e., closure systems fulfilling the property $(S4)$ which guarantees complete separation. 

Some of the earliest characterization of Kakutani closure systems was obtained by Ellis \cite{ellis1952} generalizing results of Stone \cite{Stone38} that distributive lattices and Tukey \cite{Tukey1942} that real linear spaces fulfill $(S4)$. While Bair \cite{BAIR1975696} restricts to straight line spaces and Bryant and Webster \cite{BRYANT1973321} give a very basic overview the work of Chepoi \cite{Chepoi1994} analyzes the $(S4)$ property of closure systems espacially for n-ary convexities and interval convexities (e.g. graphs with geodesic convexity and lattices) in great details. One of his main results we are using troughout our paper is the characterization of Kakutani closure systems via the Pasch Axiom.

While the above papers concentrate on the theoretical aspects of half-space separation we are also interested in the algorithmic properties and complexity results. Concerning the complexity of the general half-space separation problem as far as we know we are the first showing that it is NP-hard using the complexity result of \cite{ARTIGAS20111968} for the special case of graphs. Moreover, we relax to the more general case of maximal disjoint closed sets which allows for further applications, see~\cite{Seiffarth20}.

Other recent work focusing on the applications of half-spaces separations in computer science espacially in the domain of graphs are \textit{cluster} recovery with queries~\cite{Bressan21}, 
\textit{vertex classification} in batch~\cite{Macedo19,Seiffarth19,Stadtlander21}, \textit{active learning}~\cite{ThiessenG21} and \textit{genome rearrangement}~\cite{Cunha18}.

	\section{Preliminaries}
	\label{sec:preliminaries}
	In this section we collect the necessary notions and fix the notation for set and closure systems. For details on closure systems and separation axioms (see, e.g., \cite{Chepoi1994,Davey/Priestley/2002,Vel1993}). 
	
	\paragraph{Closure Systems}
	For a set $E$, $\pot$ denotes the power set of $E$. 
	A \emph{set system} over a ground set $E$ is a pair $(E, \cl)$ with $\mathcal{C}\subseteq\pot$; $(E, \cl)$ is a \emph{closure system} if 
	it fulfills the following two properties:
	\begin{enumerate}[label=\roman*)]
		\item 
		$E\in\cl$ and
		\item 
		$X\cap Y\in\cl$  for all $X, Y\in\cl$. 
	\end{enumerate}		
	Throughout this paper by closure systems we always mean closure systems over {\em finite} ground sets (i.e., $|E| < \infty$).
	It is a well-known fact (see, e.g.,~\cite{Davey/Priestley/2002}) that any closure system can be defined by a \emph{closure operator}, i.e., function $\cc:\pot\rightarrow\pot$ satisfying 
	the following properties for all $X, Y\subseteq E$:
	\begin{enumerate}[label=\roman*)]
		\item $X\subseteq \cc(X)$, \hfill (\textit{extensivity})
		\item $\cc(X)\subseteq \cc(Y)$ whenever $X\subseteq Y$,  \hfill (\textit{monotonicity})
		\item $\cc(\cc(X))= \cc(X)$. \hfill (\textit{idempotency})
	\end{enumerate}
	
	For a closure system $(E,\cl)$, the corresponding closure operator $\cc$ is defined by
	$$\cc(X) = \bigcap_{C \in \cl: X\subseteq C } C$$
	for all $X \subseteq E$.
	Conversely, for a closure operator $\cc$ over $E$ the corresponding closure system, denoted $(E,\cl_\cc)$, is defined by the family of its \textit{fixed points}, i.e.,
	$$\cl_\cc = \{X\subseteq E : \cc(X)=X \}.$$ 
	Depending on the context we sometimes omit the underlying closure operator from the notation and denote the closure system simply by $\clSystem$.
	The elements of $\cl_\cc$ of a closure system $(E,\cl_\cc)$ will be referred to as {\em closed} or {\em convex} sets.
	This latter terminology is justified by the fact that closed sets generalize several properties of {\em convex hulls} in $\Rd$.
	As a straightforward example, for any finite set $E \subset\mathbb{R}^d$, the set system $(E,\cl_\convexhull)$  with $\convexhull : 2^E \to 2^E$ defined by
	\begin{equation}
	\label{eq:convexhull}
	\convexhull:  X \mapsto \conv{X} \cap E
	\end{equation} 
	for all $X \subseteq E$ is a closure system, where $\conv{X}$ denotes the convex hull of $X$ in $\mathbb{R}^d$.  We will refer to this type of closure systems as \textit{$\alpha$-closure systems}.	

	\paragraph{Separation in Closure Systems} \quad	
	We now turn to the generalization of binary separation in $\Rd$ by hyperplanes to binary separation in \textit{abstract} closure systems.\footnote{For a detailed introduction into this topic (see, e.g.~\cite{Vel1993}).} In the context of \textit{machine learning}, one of the most relevant and natural questions concerning closure systems $(E, \cl)$ is whether two subsets of $E$ are separable in $\cl$, or not.
	To state the formal problem definition, we follow the generalization of half-spaces in Euclidean spaces to closure systems from \cite{Chepoi1994}.
	More precisely, let $(E, \cl)$ be a closure system. Then $H\subseteq E$ is called a \emph{half-space} in $\cl$ if both $H$ and its complement, denoted $H^c$, are closed (i.e., $H,H^c\in\cl$). Note that $H^c$ is also a half-space by definition.
	Two sets $A,B\subseteq E$ are \emph{half-space separable} if there is a half-space $H \in \cl$ such that $A\subseteq H$ and $B \subseteq H^c$; $H$ and $H^c$ together form a {\em half-space separation} of $A$ and $B$. The following property will be used many times in what follows:
	\begin{Proposition}
		\label{pr:halfspacesep}
		Let $\clSystemcc$ be a closure system, $H,H^c \in \cl$, and $A,B \subseteq E$. Then $H$ and $H^c$ form a half-space separation of $A$ and $B$ if and only if they form a half-space separation of $\cc(A)$ and $\cc(B)$.
	\end{Proposition}
	\begin{proof}
		The ``if'' direction is immediate by the extensivity of $\rho$. The ``only-if'' direction follows from the fact that for any $S \subseteq E$ and $C \in \cl_\cc$ with $S \subseteq C$ we have $\cc(S) \subseteq \cc(C) = C$ by the monotonicity and idempotency of $\cc$.
	\end{proof}
	Throughout this paper we will be concerned with half-space separation of \textit{non-empty} subsets of the ground set. Proposition~\ref{Prop:Closed} below provides a necessary condition for this problem.
	Its proof is immediate from the property that $\cC$ is closed under intersection.
	\begin{Proposition}\label{Prop:Closed}
		Let $\clSystem$ be a closure system and $A,B$ be non-empty subsets of $E$ that are half-space separable in $\cC$. Then $\cc(\emptyset) = \emptyset$.
	\end{Proposition}
	
	%
	Notice that half-space separability in abstract closure systems does not preserve all natural properties of that in $\Rd$. For example, for any two {\em finite} subsets of $\Rd$ it always holds that they are half-space separable if and only if their convex hulls\footnote{Notice that the function mapping any subset of $\Rd$ to its convex hull is a closure operator.}  
	are disjoint. 
	In contrast, as we show in Section~\ref{sec:maximalclosed} below, this equivalence does not hold for finite closure systems in general.


	\section{Half-Space and Maximal Closed Set Separation in Closure Systems}
	\label{sec:maximalclosed}
	
	In this section we first formulate some results concerning the computational complexity of the following decision problems:
	\begin{problem}[\sc The \HSP~(HSS)~Problem] 
		\emph{Given} a closure system $(E, \cl_\cc)$ with $|E| < \infty$ via  $\cc$ and non-empty subsets $A,B \subseteq E$, \emph{decide} whether $A$ and $B$ are half-space separable in $\cl_\cc$, or not.
	\end{problem}
	
	For algebraic reasons we disregard the degenerate case of $A= \emptyset$ or $B= \emptyset$.
	Furthermore, similarly to the infinite closure system over $\Rd$ defined by the family of all convex hulls in $\Rd$, we suppose that the (abstract) closure system is given \textit{implicitly}. More precisely, we assume that $\clSystemcc$ is given by the corresponding closure operator $\rho$, which returns $\rho(X)$ for any $X \subseteq E$ in \textit{unit} time. Accordingly, we characterize the complexity of the algorithms by the number of closure operator calls they require. 
	The assumption that $\cl_\cc$ is given implicitly (or intensionally) is natural, as $|\cl_\cc|$ can be exponential in $|E|$.  

	Clearly, the solution of an instance of the HSS problem is always ``{\sc No}'' whenever $\cc(A) \cap \cc(B) \neq \emptyset$.
	However, as shown in the example below, the converse of the implication is not true, i.e., the disjointness of the closures of $A$ and $B$ does not imply their half-space separability in $\cl$.
	
	\begin{figure}[t]
		\centering
		\begin{tikzpicture}[scale = 0.75]
		\centering
		\node[blue!80!black, draw, circle, fill, inner sep=1pt, label=$u$](0) at (0, 0) {};
		\node[blue!80!black, draw, circle, fill, inner sep=1pt, label=below:$v$](1) at (4, 0) {};

		\node[black, draw, circle, fill, inner sep=1pt, label=$y$](2) at (2, 1.2) {};
		\node[black, draw, circle, fill, inner sep=1pt, label=below:$w$](3) at (2, -1.2) {};
		\node[red!80!black, draw, circle, fill, inner sep=1pt, label=$x$](4) at (3.8, 0.3) {};
		
		\node[draw, circle, fill, inner sep=1pt, label=right:$z$](5) at (8, 1.2) {};
		
		\draw[dashed, blue!80!black] (0)--(1);
		
		\draw[dashed, red!80!black] (2)--(3);
		\draw[dashed, red!80!black] (3)--(4);
		\draw[dashed, red!80!black] (2)--(4);
		
		\draw[dotted, blue!80!black] (0)--(5);
		\draw[dotted, blue!80!black] (1)--(5);
		
		\draw[dotted, red!80!black] (2)--(5);
		\draw[dotted, red!80!black] (3)--(5);
		\end{tikzpicture}
		\caption{Example of a point configuration in $\mathbb{R}^2$, where adding $z$ to either of the closed sets $\{x,y,w\}$ and $\{u,v\}$ would violate the disjointness condition.}
		\label{fig:nonkakutani}	
	\end{figure}
	
	\begin{Example}
		\label{ex:nonkakutaniA}
		Consider 
		the set $E \subset \mathbb{R}^2$ consisting of the six points in Fig.~\ref{fig:nonkakutani} and the $\alpha$-closure system $(E,\cC_\convexhull)$ defined in (\ref{eq:convexhull}).
		Though $\{u,v\}$ and $\{x,y,w\}$ are both closed (i.e., belong to $\cC_\convexhull$) and disjoint, they are not half-space separable in $\cC_\convexhull$, as $z$ can be added to neither of the sets without violating the disjointness property of half-space separation.
	\end{Example}
	
	This difference to $\Rd$ makes, among others, the more general problem setting considered in this work computationally difficult, as shown in Thm.~\ref{pr:HSP} below.	
	The fact that the disjointness of $\cc(A)$ and $\cc(B)$ does not imply half-space separability of $A$ and $B$ makes the HSS~problem computationally intractable.
	To prove this negative complexity result, we adopt the definition of {\em convex} vertex sets of a graph defined by shortest paths~\cite{Harary/Nieminen/1981,Mulder/1980}, also referred to as the \textit{geodesic graph closure}. More precisely, for an undirected graph $G =(V, E)$ we consider the set system $(V,\cC_{\cg})$ with
	\begin{equation}
	\label{eq:clgamma}
	V'\in \cC_{\cg} \iff \forall u,v \in V', \forall P \in \mathcal{S}_{u,v} : V(P) \subseteq V'
	\end{equation}
	for all $V' \subseteq V$, where $\mathcal{S}_{u,v}$ denotes the set of shortest paths connecting $u$ and $v$ in $G$ and $V(P)$ the set of vertices in $P$.
	Notice that $(V,\cC_{\cg})$ is a closure system. This follows directly from the fact that the intersection of any two convex subsets of $V$ is also convex, by noting that the empty set is also convex by definition. This type of closure systems will be referred to as \textit{$\cg$-closure systems} throughout this paper. Using the above definition of graph convexity, we consider the following problem:
	
	\begin{description}
		\item[\sc Convex 2-Partitioning Problem:] {\em Given} an undirected graph $G =(V, E)$, {\em decide} whether there is a {\em proper} partitioning of $V$ into two convex sets.
	\end{description}
	This problem is known to be NP-complete~\cite{ARTIGAS20111968}.
	Notice that the condition on properness is necessary, as otherwise $\emptyset$ and $V$ would always form a (trivial) solution.
	Note also the difference between the HSS and the {\sc Convex 2-Partitioning} problems. The latter one is concerned with a property of $G$ (i.e., has no additional input $A,B$). 
	The concepts and result above imply the following negative result (see Appendix~\ref{app:A} for the proof):	
	\begin{Theorem}
		\label{pr:HSP}
		The {\sc HSS}~problem is $\operatorname{NP}$-hard.
	\end{Theorem}

We will refer to the problem defined in the same way as the {\sc HSS}~problem, but with the difference that $\cc$ can be computed in time $\bigO{p(|V|)}$-time for some  polynomial $p$ as the {\sc HSSPOLY}-problem.
Thm.~\ref{pr:HSP} immediately implies the following negative results:
	\begin{Corollary}The {\sc HSSPOLY} problem is NP-complete.
	\end{Corollary}
The next corollary is concerned with the complexity of computing a closed set separation of \textit{maximum} size. (Note that the case of $k = |E|$ in the corollary below corresponds to the {\sc HSS} problem.)
	\begin{Corollary}
		\label{cor:MS}
		Given a closure system $(E, \cl_\cc)$ via $\cc$ as in the HSS~problem, non-empty subsets $A,B \subseteq E$, and an integer $k > 0$, it is NP-hard to decide whether there are disjoint closed sets $H_1,H_2 \in \cl_\cc$ with $A \subseteq H_1$, $B \subseteq H_2$ such that $|H_1| + |H_2| \geq k$.
	\end{Corollary}
	%
	%
	
	The negative results above motivate us to relax the {\sc HSS} problem.
	
	\subsection{Maximal Closed Set Separation}
	\label{sec:MCSep}
	One way to overcome the negative results formulated in Thm.~\ref{pr:HSP} and Corollary~\ref{cor:MS} is to weaken the condition on half-space separability in the HSS~problem to the problem of {\em maximal} closed set separation:
	
	\begin{description}
		\item[\sc Maximal Closed Set Separation (\MCSep) Problem:] 
		\emph{Given} a closure system $(E, \cl_\cc)$ via $\cc$ as in the HSS~problem and non-empty subsets $A,B \subseteq E$, \textit{find} two disjoint closed sets $H_1,H_2 \in \cl_\cc$ with $A\subseteq H_1$ and $B\subseteq H_2$ such that there are no disjoint sets $H_1', H_2'\in \cl_\cc$ with $H_1\subseteq H_1'$ and $H_2\subseteq H_2'$, where at least one of the containments is proper; o/w return ``NO'' (i.e., if such closed sets do not exist).
	\end{description}
	
\begin{algorithm}[t]
	\SetAlgoLined
	\caption{\sc Maximal Closed Set Separation (\AlgoSep)}
	\DontPrintSemicolon
	\LinesNumbered
	\label{AlgoGreedyHalfSep}
	\KwIn{finite closure system $(E,\cl_\cc)$ given by $\cc$ and $A,B \subseteq E$ with $A,B \neq \emptyset$}
	\KwOut{{\em maximal} disjoint closed sets $H_1,H_2 \in \cl_\cc$ with $A  \subseteq H_1$ and $B \subseteq H_2$ if $\cc(A) \cap \cc(B) = \emptyset$; ``{\sc No}'' o/w}
	\BlankLine
	$H_1\leftarrow \cc(A)$, $H_2\leftarrow \cc(B)$\;
	\If{$H_1 \cap H_2 \neq \emptyset$}{\Return{``{\sc No}''}}
	$F\leftarrow E\setminus (H_1\cup H_2)$\;
	\While{$F\not = \emptyset$}{
		{\label{line:chooseElements}}choose $e\in F$ and remove it from $F$\;
		\uIf{$\cc(H_1\cup\{e\})\cap H_2 = \emptyset$}
		{\label{LineStatement1}
			$H_1\leftarrow \cc(H_1\cup \{e\})$, \label{LineH1Set}
			$F\leftarrow F\setminus H_1$\label{LineFReduce1}}	
		\ElseIf{ $\cc(H_2\cup\{e\})\cap H_1 = \emptyset$}
		{\label{LineStatement2}
			$H_2\leftarrow \cc(H_2\cup\{e\})$, \label{LineH2Set}
			$F\leftarrow F\setminus H_2$\label{LineFReduce2}}
		
	}
	\Return{$H_1, H_2$}\;
\end{algorithm}
	
	In this section we present Alg.~\ref{AlgoGreedyHalfSep}, that solves the \MCSep~problem and is \textit{optimal} w.r.t. the worst-case number of closure operator calls. 
	Alg.~\ref{AlgoGreedyHalfSep} takes as input a closure system $(E,\cl_\cc$) over some finite ground set $E$,  where $\cl_\cc$ is given via the closure operator $\cc$, and non-empty sets $A,B \subseteq E$. If the closures of $A$ and $B$ are not disjoint, then it returns ``NO'' (cf. Lines~1--3). Otherwise, the algorithm tries to extend one of the largest closed sets $H_1 \supseteq A$ and $H_2 \supseteq B$ found so far consistently by an element $e\in F$, where $F= E\setminus (H_1\cup H_2)$ is the set of potential generators. By consistency we mean that the closure of the extended set must be disjoint with the other (unextended) one (cf. Lines~8 and 10). 
	Note that each element will be considered at most once for extension (cf. Line~7). If $H_1$ or $H_2$ could be extended, then $F$ will correspondingly be updated (cf. Lines~9 and~11), by noting that $e$ will be removed from $F$ even in the case it does not result in an extension (cf. Line~7). The algorithm repeatedly iterates the above steps until $F$ becomes empty; at this stage it returns $H_1$ and $H_2$ as a solution.
	In the theorem below concerning some basic properties of Alg.~\ref{AlgoGreedyHalfSep}, $\gen(C,X)$ denotes the cardinality of a \textit{smallest} relativ generator set $G \subseteq E$ such that $\cc(X\cup G)=C$, for $C \in \cC$ and $X \subseteq C$.
	
	\begin{Theorem}\label{thm:correctness}
		Alg.~\ref{AlgoGreedyHalfSep} is correct. Furthermore, for the number $N$ of its closure operator calls it holds that $N=2$ if $\cc(A) \cap \cc(B) \neq \emptyset$; o/w 
		\begin{equation}
			\gen(E, A\cup B) +2 \leq N \leq 2|E| - 2 \enspace . \label{eq:bounds}
		\end{equation}
	\end{Theorem} 
\begin{proof}
The correctness is straightforward by noting that the maximality of the output closed sets $H_1$ and $H_2$ follows from the monotonicity of $\cc$, as all elements $e$ considered by the algorithm and not added earlier to one of the closed sets (cf. Lines~9 and 11) can be added later neither to $H_1$ nor to $H_2$ without violating the disjointness.

Regarding the second part of the claim, it is trivial for the case that $\cc(A) \cap \cc(B)\neq \emptyset$, so assume $\cc(A) \cap \cc(B)= \emptyset$.
For the lower bound in (\ref{eq:bounds}), note that 
\begin{align}\label{eq:LowerBoundBasis}
	N\geq 2 + \gen(H_1, A) + \gen(H_2, B) + m \enspace ,
\end{align}
where $m = |E \setminus (H_1\cup H_2)|$.
We  claim that
\begin{align}\label{eq:LowerBoundGenerator}
	\gen(H_1, A) + \gen(H_2, B) \geq \gen(\cc(H_1\cup H_2), A\cup B) \enspace.
\end{align}
Indeed, by definition there are  $G_1,G_2 \subseteq E$ with $|G_1| = \gen(H_1,A)$ and $|G_2| = \gen(H_2,B)$ such that $H_1=\cc(A\cup G_1)$ and $H_2=\cc(B\cup G_2)$. Moreover, 
\begin{equation*}
	\cc(H_1 \cup H_2) = \cc(\cc(A\cup G_1)\cup\cc(B\cup G_2)) \subseteq 
	\cc(A \cup B \cup G_1 \cup G_2)
\end{equation*}
implying $|G_1 \cup G_2| \geq\gen(\cc(H_1\cup H_2),A\cup B)$, from which (\ref{eq:LowerBoundGenerator}) follows by $|G_1| + |G_2| \geq |G_1 \cup G_2|$.
We now show that
\begin{align}\label{eq:LowerBoundE}
	\gen(\cc(H_1\cup H_2), A\cup B) + m \geq\gen(E, A\cup B).
\end{align}
By definition, there exists $G \subseteq E$ with $|G| = \gen(\cc(H_1\cup H_2), A\cup B)$ such that $\cc(H_1\cup H_2) = \cc(A\cup B\cup G)$.
But then, for $X=E\setminus(H_1\cup H_2)$ we have
\begin{align*}
	E 
	&= H_1\cup H_2\cup X \\
	&\subseteq \cc(H_1\cup H_2)\cup X \\
	& = \cc(A\cup B\cup G) \cup X \\
	&\subseteq \cc(A\cup B\cup G\cup X)
\end{align*}
from which we have (\ref{eq:LowerBoundE}) by $|X|=m$.
The lower bound in (\ref{eq:bounds}) then follows from (\ref{eq:LowerBoundBasis})--(\ref{eq:LowerBoundE}).

		Regarding the upper bound in (\ref{eq:bounds}), the algorithm calls initially the closure operator twice (cf. Line~1) and then at most twice per iteration (cf. Lines~9 and 11), giving 
		$$
		 N \leq 2\cdot|E\setminus(\cc(A)\cup\cc(B))| + 2 \enspace .
		$$
		The claim then follows from the case that $A$ and $B$ are closed singletons. 
	\end{proof}
	We stress that Alg.~\ref{AlgoGreedyHalfSep} has access to $\clSystemcc$ only via $\cc$, i.e., it does not utilize any domain specific properties. 
	The following example shows that, under this assumption, the number of closure operator calls depends on the order of $A$ and $B$ as well as on that of the elements selected in Line~7.
	
	\begin{Example} 
		\label{ex:greedyworstcase}
		Let $\clSystemcc$ be the closure system with $E = \{1, 2, \ldots, n\}$ for some integer $n > 1$ and with the corresponding closure operator defined by $\cc: X \mapsto \{x \in E: \min X \leq x \leq \max X \}$ for all $X \subseteq E$.
		Consider first the case that $A = \{2\}$, $B = \{1\}$, and $n$ has been chosen in Line~7 for the first iteration. For this case, the algorithm terminates after the first iteration returning the closed half-spaces $H_1 = \{2,\ldots,n\}$ and $H_2 = \{1\}$, and  calling the closure operator together three times. 
		
		Now consider the case that $A = \{1\}$, $B = \{2\}$, and the elements in Line~7 are processed in the order $3,4,\ldots,n$. One can easily check that the algorithm returns the same two half-spaces after $n-2$ iterations. 
		In each iteration it calls the closure operator twice, giving together the worst-case upper bound $2n-2$ claimed in Thm.~\ref{thm:correctness}.  
	\end{Example}
	
	Though the example above may suggest that Alg.~\ref{AlgoGreedyHalfSep} is not optimal, the worst-case upper bound stated in Thm.~\ref{thm:correctness} is in fact the best possible, regardless of the order of the elements in Line~7 for the general case stated above. Furthermore, the example indicates that using domain specific structure might help to reduce the number of closure operator calls drastically. This is verified in case of lattices in Section~\ref{sec:Lattices}. 
	To show optimality in the general case, we first prove the following lemma.
	
	\begin{lemma}\label{lm:optimal}
		All algorithms solving the \MCSep~problem require at least $2|E|-2$ closure operator calls in the worst-case.
	\end{lemma}
	\begin{proof}
		Suppose for contradiction that there is an 
		algorithm $\mathfrak{A}$ solving the \MCSep ~\-problem with strictly less than $2|E|-2$ closure operator calls for \textit{all} problem instances.
		Consider the closure system $(E, \cl_\cc)$ with $E=\{e_1,\ldots, e_n\}$ for some $n>2$ and with the closure operator $\rho$ defined by
		$$
		\rho(X) =
		\begin{cases}
		X & \text{if $X \in \{ \emptyset, \{e_1\}, \{e_2\}\}$} \\
		E & \text{o/w .}	
		\end{cases}
		$$
		By condition, $\mathfrak{A}$ returns the only solution $\{e_1\}$ and $\{e_2\}$ of the \MCSep~problem for the input $\{e_1\}$ and $\{e_2\}$ with at most $2|E|-3$ closure operator calls.
		We claim that $\mathfrak{A}$ needs to calculate the closure for both input sets. Indeed, suppose the closure of one of them, say $\{e_1\}$, has not been computed. Then $\mathfrak{A}$ is incorrect, as it would return the same output for the closure system above and for $(E, \cl_\cc\setminus\{e_1\})$.
		Thus, $\mathfrak{A}$ can calculate the closure for at most $2|E|-5$ further subsets of $E$.
		This implies that the closure has not been considered by $\mathfrak{A}$ for at least one of the sets $\{e_1, e_3\}, \ldots, \{e_1, e_n\},\{e_2, e_3\},\ldots,\{e_2, e_n\}$, say for $\{e_1, e_3\}$. 
		But then 
		$\mathfrak{A}$ returns the same output for $(E, \cl_\cc)$ and for the closure system $(E, \cl_\cc \cup \{e_1, e_3\})$, contradicting its correctness.
	\end{proof}
	
	\begin{Theorem}\label{thm:optimal} Alg.~\ref{AlgoGreedyHalfSep} is optimal w.r.t.~the worst-case number of closure operator calls.
	\end{Theorem}
	\begin{proof}
		It is immediate from the upper bound in Thm.~\ref{thm:correctness} and Lemma~\ref{lm:optimal}.
	\end{proof}

	As a second relaxation besides maximal disjoint closed sets, in Section~\ref{sec:kakutani} we consider \textit{Kakutani} closure systems, a special kind of closure systems, for which a half-space separation always exists if the closures of the input sets are disjoint.
	We will show that for this type of closure systems, Alg.~\ref{AlgoGreedyHalfSep} provides an algorithmic characterization and solves the HSS~problem correctly and efficiently. 
	
	\section{Kakutani Closure Systems}
	\label{sec:kakutani}
	A natural way to overcome the negative result stated in Thm.~\ref{pr:HSP} is to consider closure systems in which {\em any} two disjoint closed sets are half-space separable. 
	More precisely, for a closure operator $\clop$ over a ground set $E$, the corresponding closure system $(E,\cl_\clop)$ is {\em Kakutani}\footnote{A similar property was considered by the Japanese mathematician Shizou Kakutani for Euclidean spaces cf.~\cite{Kakutani37}} if it fulfills the {\em $S_4$ separation axiom} defined as follows: For all $A,B \subseteq E$, 
	$$
	\text{$A$ and $B$ are half-space separable in $(E,\cl_\clop)$ $\iff$ $\clop(A)\cap \clop(B)=\emptyset$} \enspace .
	$$
	For a detailed reference on closure systems satisfying the $S_4$ separation property, the reader is referred to \cite{Chepoi1994}.
	By Proposition~\ref{pr:halfspacesep}, any half-space separation of $A,B$ in $\cl_\clop$ is a half-space separation of $\clop(A)$ and $\clop(B)$ in $\cl_\clop$.  
	Clearly, the HSSPOLY~problem can be decided in polynomial time for Kakutani closure systems: For any $A,B \subseteq E$ just calculate $\clop(A)$ and $\clop(B)$ and check whether they are disjoint, or not.
	Furthermore, if $A$ and $B$ are half-space separable, $\clSystemcc$ is Kakutani, and $\cc$ can be computed in time polynomial in $n$, then Alg.~\ref{AlgoGreedyHalfSep} returns a half-space separation of $A, B$ in polynomial time. 
	
	
	\begin{Example}
		\label{ExampleNonKakutani} 
		The closure system used in Example~\ref{ex:nonkakutaniA} is \textit{not} Kakutani.
		For an example of Kakutani closure systems, consider an arbitrary non-empty finite subset $E \subset \mathbb{R}^2$ of a circle and define the set system $\cC \subseteq 2^E$ as follows: For all $E' \subseteq E$, $E' \in \cC$ if and only if there exits a closed half-plane $H \subseteq \mathbb{R}^2$ satisfying $E' = H \cap E$. One can easily check that $(E,\cC)$ is a Kakutani closure system.		
		%
		\qed
	\end{Example}

	In the theorem below we show that Alg.~\ref{AlgoGreedyHalfSep}, besides solving the \MCSep~problem, also provides an {\em algorithmic characterization} of Kakutani closure systems. 
	
	\begin{Theorem}
		\label{thm:characterization}
		Let $(E, \cl_\cc)$ be a closure system with corresponding closure operator $\cc$.
		Then $(E, \cl_\cc)$ is Kakutani if and only if  for all  non-empty $A,B \subseteq E$ with $\cc(A)\cap \cc(B)=\emptyset$, the output of Alg.~\ref{AlgoGreedyHalfSep} is a partitioning of $E$. 
	\end{Theorem}
	
	\begin{proof}
		The sufficiency is immediate by Thm.~\ref{thm:correctness} and the definition of Kakutani closure systems.
		For the necessity, let $(E, \cl_\cc)$ be a Kakutani closure system.
		It suffices to show that for all $e \in F$ selected in Line~7 of Alg.~\ref{AlgoGreedyHalfSep}, $e$ is always added to one of $H_1$ or $H_2$; the claim then follows by Thm.~\ref{thm:correctness} for this direction. Suppose for contradiction that there exists an $e \in F$ selected in Line~7 that can be used to extend neither of the closed sets $H_1,H_2$. Since $H_1$ and $H_2$ are disjoint closed sets and $(E,\cl_\cc)$ is Kakutani, there are $H_1',H_2' \in \cl_\cc$ such that $H_1 \subseteq H_1'$, $H_2 \subseteq H_2'$, and $H_2' = (H_1')^c$.
		Hence, exactly one of $H_1'$ and $H_2'$ contains $e$, say $H_1'$.
		By the choice of $e$, $\cc(H_1 \cup\{e\}) \cap H_2 \neq \emptyset$. Since $\cc$ is monotone, $\cc(H_1 \cup\{e\}) \subseteq H_1'$ and hence $H_1'$ and $H_2'$ are not disjoint; a contradiction.	
	\end{proof}
	
	The characterization result formulated in Thm.~\ref{thm:characterization} cannot, however, be used to decide in time polynomial in $|E|$, whether an \textit{intensionally} given closure system $(E,\cl_\cc)$ is Kakutani, or not. More precisely, in Thm.~\ref{thm:kakutaninegative} below we have a negative result for the following decision problem:
	\begin{description}
		\item[\sc Kakutani Problem:] {\em Given} a closure system $(E,\cl_\cc)$, where $\cl_\cc$ is given by the corresponding closure operator $\cc$, {\em decide} whether $(E,\cl_\cc)$ is Kakutani, or not.
	\end{description}
	\begin{Theorem}
		\label{thm:kakutaninegative}
		Any algorithm solving the Kakutani problem above requires $\Omega\left(2^{|E|/2}\right)$ closure operator calls.
	\end{Theorem}	
	\begin{proof}
		We can assume w.l.o.g. that $\emptyset \in \cl_\cc$, as otherwise there are no two separable subsets of $E$.
		For any even number\footnote{A similar proof applies to odd numbers. For simplicity, we omit the discussion of that case.} 
		$n \in \mathbb{N}$, consider a set $E$ with $|E| = n$ and the set system 
		$$			
		\cl_\cc = \left\{X\subseteq E : |X| \leq {n}/{2} \right\}\cup\{E\} \enspace .
		$$
		We claim that $(E, \cl_\cc)$ is a Kakutani closure system. Since  $\emptyset,E \in\cl_\cc$ and $|C_1 \cap C_2| \leq n/2$ for any $C_1,C_2 \in \cl_\cc$, $(E,\cl_\cc)$ is closed under intersection and hence, it is a closure system.
		To see that it is Kakutani, notice that all $X \in \cl_\cc$ with $|X| = n/2$ are half-spaces; all other closed sets $Y \in \cl_\cc$ with $0 < |Y| < n/2$ are not half-spaces.
		Thus, for any non-empty $A,B \subseteq E$ with $\cc(A) \cap \cc(B) = \emptyset$, $\cc(A)$ can be extended to a half-space $H_1 \in \cl_\cc$ such that $H_1 \cap \cc(B) = \emptyset$. By construction, $H_1$ and its complement $H_1^c$ form a half-space separation of $A$ and $B$. Hence, $(E,\cl_\cc)$ is Kakutani.
		Note also that for any $C \in \cl_\cc$ with $|C| = n/2$, $(E,\cl_\cc \setminus \{C\})$ remains a closure system, but becomes non-Kakutani.
		
		We are ready to prove the lower bound claimed. Suppose for contradiction that there exists an 
		algorithm $\mathfrak{A}$ that decides the Kakutani problem with strictly less than ${n\choose n/2} = \Omega\left(2^{n/2}\right)$ closure operator calls.  
		Then, for $(E,\cl_\cc)$ above, there exists a half-space $C \in \cl_\cc$ with $|C| = n/2$ such that $\mathfrak{A}$ has not called $\cc$ for $C$. But then $\mathfrak{A}$ returns the same answer for the Kakutani and non-Kakutani closure systems $(E,\cl_\cc)$ and $(E,\cl_\cc \setminus \{C\})$, contradicting its correctness.    
	\end{proof}

	The exponential lower bound in Theorem~\ref{thm:kakutaninegative} above holds for {\em arbitrary} (finite) closure systems.
	Fortunately, there is a broad class of closure systems that are known to be Kakutani.
	In particular, as a generic application field of Kakutani closure systems, in Section~\ref{sec:kakutanigraphs} we focus on Kakutani closure systems over \textit{graphs} and in Section~\ref{sec:LatticesKakutani} on those over finite \textit{lattices}.

\section{Domain Specific Results}
\label{sec:domainspecific}

This section is concerned with the special cases of closure systems over \textit{graphs} and \textit{lattices}.

\subsection{Closure Systems over Graphs}
\label{sec:kakutanigraphs}

As a first application of Theorem~\ref{thm:characterization}, in this section we consider Kakutani closure systems over \textit{graphs}.
The following result provides a characterization of Kakutani $\cg$-closure systems over graphs in terms of the \textit{Pasch axiom}~\cite{Chepoi1994}.
We note that the theorem below holds for closure systems defined by any set of paths in the underlying graph. 
\begin{Theorem}\cite{Chepoi1994}
	\label{thm:paschkakutani}
	For any finite graph $G=(V,E)$, the corresponding $\cg$-closure system $\vclog$ is Kakutani if and only if $\cg$ fulfills the Pasch axiom, i.e., 
	\begin{equation}
	\label{eq:pasch}
	x \in \cg(\{u, v\}) \wedge y \in \cg(\{u, w\})	
	\implies  \cg(\{x, w\})\cap \cg(\{y, v\})\neq \emptyset
	\end{equation}
	for all $u, v, w, x, y\in V$.
\end{Theorem}
Note that the theorem above can be turned into a na\"{\i}ve algorithm that decides the Kakutani problem for $\cg$-closure systems over graphs in $\bigO{n^8}$ time and in $\bigO{n^2}$ space by checking the condition in (\ref{eq:pasch}) for all quintuples of vertices; the complexity of computing $\cg$ for any set of vertices is $\bigO{n^3}$~\cite{Dourado/etal/2009}. 
The following more sophisticated algorithm\footnote{Personal communication with Victor Chepoi.} reduces this time complexity to $\bigO{n^5}$ time, using, however, $\bigO{n^4}$ space.
Compute first $\cg(\{u, v\})$ for all $u, v \in V$ and store them in a matrix $M$.  
The time and space complexity of this step is $\bigO{n^5}$ and $\bigO{n^3}$, respectively. 
Using $M$, for each pair of closed sets we can decide in $\bigO{n}$ time whether or not they are disjoint and store this information in a binary matrix $B$. This can be done in $\bigO{n^5}$ time and $\big(n^4)$ space.
Iterating over all quintuples $Q=(u, v, w, x, y)\in V^5$, we can check in constant time from $M$ and $B$ whether $Q$ fulfills (\ref{eq:pasch}), implying the time and space complexity claimed above.

Using the characterization result above, in the theorem below we give a sufficient condition in terms of \textit{forbidden minors} for the Kakutani property for $\cg$-closure systems.
More precisely, as the main contribution of this section, we show that a closure system of a graph is Kakutani whenever the underlying graph does not contain $K_{2,3}$ as a minor. This result may be of some independent interest as well. 
%
To state the theorem, we recall that $K_{2, 3}$ denotes the complete bipartite graph $(V_1,V_2,E)$ with $|V_1| = 2$ and $|V_2| = 3$. Furthermore, a graph $H$ is a \textit{minor} of a graph $G$ if $H$ can be obtained from $G$ by a sequence
of vertex and edge deletions and edge contractions (see, e.g., \cite{Diestel/2012}).

\begin{Theorem}
	\label{thm:outerplanar} 
	For any finite graph $G=(V,E)$, the  $\cg$-closure system $\vclog$ is Kakutani if $G$ does not contain $K_{2, 3}$ as a minor.
\end{Theorem}
\begin{proof}
	We prove the claim by contraposition. 
	Let $G = (V, E)$ be a graph such that $\vclog$ is \textit{not} Kakutani. 
	Then, by Theorem \ref{thm:paschkakutani}, $\cg$ does not fulfill the Pasch axiom, i.e., there are $u, v, w\in V$, $x\in \cg(\{u, v\})$, and $y\in \cg(\{u, w\})$ such that
	\begin{align}
		\label{Line:PaschNegation}
		\cg(\{v, y\})\cap \cg(\{x, w\}) = \emptyset.
	\end{align}
	We claim that $u, v, w, x, y$ are pairwise different. 
	We show this property only for $x$ and $w$; the proofs for the other vertex pairs are similar.	
	Suppose for contradiction that $x=w$. Then, by (\ref{Line:PaschNegation}), $x$ (resp. $y$) lies on a shortest path between $u$ and $v$ (resp. $u$ and $x$). 
	But then there is a shortest path between $v$ and $y$ containing $x$ (i.e., $x \in \cg(\{v,y\})$), contradicting the disjointedness condition in (\ref{Line:PaschNegation}).
	
	We are ready to show that (\ref{Line:PaschNegation}) implies that $G$ contains $K_{2, 3}=(V_1,V_2,E)$ as a minor with $V_1 = \{x,y\}$ and $V_2= \{u, v, w\}$.
	By (\ref{Line:PaschNegation}), there are shortest paths $P_{uv}$ between $u, v$ with $x\in P_{uv}$ and $P_{uw}$ between $u, w$ with $y\in P_{uw}$. 
	Let $u'$ be the common vertex of $P_{ux}$ and $P_{uy}$ that has the maximum distance to $u$.
	It must be the case that $u' \neq x$, $u' \neq y$, and $x$ (resp. $y$) lies on the subpath $P_{u'v}$ of $P_{uv}$ (resp. $P_{u'w}$ of $P_{uw}$), as otherwise the disjointness condition in (\ref{Line:PaschNegation}) does not hold. 
	Moreover, as $\cg(\{v, y\})\cap \cg(\{x, w\}) = \emptyset$ by (\ref{Line:PaschNegation}), the subpaths $P_{xv}$ of $P_{u'v}$ and $P_{yw}$ of $P_{u'w}$ must be vertex disjoint. 		
	Hence, $G$ contains the subgraph depicted in Fig.~\ref{graph:1} (it suffices to consider only one shortest path between $u'$ and $u$). 
	Regarding the shortest paths between $v$ and $y$, note that none of them can contain $u'$, as otherwise $x \in \cg(\{v, y\})$ and thus, $\cg(\{v, y\})\cap \cg(\{x, w\})$ would be  non-empty. Furthermore, by (\ref{Line:PaschNegation}), they cannot contain $x$ and $w$. 
	In a similar way, none of the shortest paths between $x$ and $w$  contains $u', y, v$. Combining these properties with the one implied by (\ref{Line:PaschNegation}) that all shortest paths between $v$ and $y$ are pairwise vertex disjoint with all shortest path between $w$ and $x$, we have that $G$ contains the subgraph given in Fig.~\ref{graph:2}. The minor $K_{2, 3}$ claimed in the theorem is then obtained from this subgraph by edge contraction (cf. Fig. \ref{graph:3}).
	\begin{figure}[t]\centering
		\subfloat[\label{graph:1}]{
			\begin{tikzpicture}[scale = 0.3]
				\centering
				\node[black, draw, circle, fill, inner sep=1pt, label= below:$w$](5) at (2, -5) {};
				\node[black, draw, circle, fill, inner sep=1pt, label= above:$u$](0) at (0, 2) {};
				\node[black, draw, circle, fill, inner sep=1pt, label=below:$v$](4) at (-2, -5) {};
				\node[black, draw, circle, fill, inner sep=1pt, label= left:$x$](2) at (-2, -2) {};
				
				\node[draw, circle, fill, inner sep=1pt, label=right:$u'$](1) at (0, 0) {};
				\node[draw, circle, fill, inner sep=1pt, label=right:$y$](3) at (2, -2) {};
				
				\draw[decorate, decoration={snake, segment length=2mm, amplitude=.4mm}]  (0)--(1);
				\draw[decorate, decoration={snake, segment length=2mm, amplitude=.4mm}]  (1)--(3);
				\draw[decorate, decoration={snake, segment length=2mm, amplitude=.4mm}]  (1)--(2);
				\draw[decorate, decoration={snake, segment length=2mm, amplitude=.4mm}]  (2)--(4);
				\draw[decorate, decoration={snake, segment length=2mm, amplitude=.4mm}]  (3)--(5);			
		\end{tikzpicture}}\hfil
		\subfloat[\label{graph:2}]{
			\begin{tikzpicture}[scale = 0.3]
				\centering
				\node[black, draw, circle, fill, inner sep=1pt, label= below:$w$](5) at (2, -5) {};
				\node[black, draw, circle, fill, inner sep=1pt, label=below:$v$](4) at (-2, -5) {};
				\node[black, draw, circle, fill, inner sep=1pt, label= left:$x$](2) at (-2, -2) {};
				
				\node[draw, circle, fill, inner sep=1pt, label=right:$u$](0) at (0, 2) {};
				
				\node[draw, circle, fill, inner sep=1pt, label=right:$u'$](1) at (0, 0) {};
				\node[draw, circle, fill, inner sep=1pt, label=right:$y$](3) at (2, -2) {};
				\node (6) at (-2, -7) {};

				\draw[decorate, decoration={snake, segment length=2mm, amplitude=.4mm}]  (0)--(1);
				\draw[decorate, decoration={snake, segment length=2mm, amplitude=.4mm}]  (1)--(3);
				\draw[decorate, decoration={snake, segment length=2mm, amplitude=.4mm}]  (1)--(2);
				\draw[decorate, decoration={snake, segment length=2mm, amplitude=.4mm}]  (2)--(4);
				\draw[decorate, decoration={snake, segment length=2mm, amplitude=.4mm}]  (3)--(5);
				\draw[decorate, decoration={snake, segment length=2mm, amplitude=.4mm}]  (3)--(4) node {};	
				\draw[decorate, decoration={snake, segment length=2mm, amplitude=.4mm}]   (2) to[out=-160,in=-180] (-2, -7);	
				\draw[decorate, decoration={snake, segment length=2mm, amplitude=.4mm}]   (-2, -7) to[out=0,in=-180] (5);
		\end{tikzpicture}}\hfil
		\subfloat[\label{graph:3}]{
			\begin{tikzpicture}[scale = 0.3]
				\centering
				\node[black, draw, circle, fill, inner sep=1pt, label= below:$w$](5) at (3, -5) {};
				\node[black, draw, circle, fill, inner sep=1pt, label= below:$u$](0) at (0, -5) {};
				\node[black, draw, circle, fill, inner sep=1pt, label=below:$v$](4) at (-3, -5) {};
				\node[black, draw, circle, fill, inner sep=1pt, label= left:$x$](2) at (-2, 0) {};
				
				\node[draw, circle, fill, inner sep=1pt, label=right:$y$](3) at (2, 0) {};

				\draw  (0)--(2);
				\draw  (2)--(4);
				\draw  (2)--(5);
				\draw  (3)--(4);
				\draw  (3)--(5);			
				\draw  (3)--(0);			
				
		\end{tikzpicture}}
		\caption{Graph minor of Non-Kakutani graphs}
	\end{figure}
\end{proof}

\begin{remark}
	We note that the converse of Theorem~\ref{thm:outerplanar} does \textit{not} hold, implying that $K_{2,3}$ as a forbidden minor does not characterize the Kakutani property for closure systems over graphs. Indeed, for all complete graphs $K_n = (V,E)$, the corresponding closure system $(V, 2^V)$ is Kakutani.
	The claim then follows by noting that $K_{2,3}$ is a minor of $K_n$ for all $n \geq 5$.
\end{remark}	

In Corollary~\ref{cor:outerplanar} below we formulate an immediate implication of Theorem~\ref{thm:outerplanar}. We recall that a graph is \textit{outerplanar} if it can be embedded in the plane such that there are no two edges crossing in an interior point and all vertices lie on the outer face. 
\begin{Corollary} 
	\label{cor:outerplanar}
	For any outerplanar graph $G=(V,E)$, the corresponding $\cg$-closure system $\vclog$ is Kakutani.
\end{Corollary}
\begin{proof} 
	It follows directly from Theorem~\ref{thm:outerplanar} together with the characterization result of outerplanar graphs by \cite{Chartrand/Harary/1967}.
\end{proof}
Note that the corollary above applies also to trees, as they are (special) outerplanar graphs. Though the result for trees is well-known, it is typically derived directly from the Pasch axiom. In contrast, we obtain it as an immediate consequence of Theorem~\ref{thm:outerplanar}.

\subsection{Closure Systems over Lattices}
	\label{sec:Lattices}
	Our second application field is concerned with closure systems over {\em lattices}.
	The focus lies, as before, on the HSS and MCSS problems for the special case that the underlying ground set is some finite lattice and the closure operator for a subset $S$ of the ground set is defined by the set of all elements lying between the infimum and supremum of $S$.
	For this kind of closure systems we give Alg.~\ref{AlgoLatticeSep} which improves Alg.~\ref{AlgoGreedyHalfSep} for lattice structures.
	Assuming that the closures of the input sets $A$ and $B$ to be separated are disjoint, Alg.~\ref{AlgoLatticeSep} extends them into a disjoint \textit{maximal ideal} $I$ and a \textit{maximal filter} $F$ such that $A \subseteq I$ and $B \subseteq F$ or vice versa  (see Fig.~\ref{fig:lattice_example} for an example of different separations of a finite lattice).
	This specialized version has some important advantages over Alg.~\ref{AlgoGreedyHalfSep}.
	In particular, for certain problem classes it reduces the number of closure operator calls \textit{logarithmically}.
	This is the situation e.g.~in frequent closed itemset mining~\cite{Pasquier_etal/99} or formal concept analysis~\cite{Gan05}.
	Furthermore, the disjointness of the closures of any two sets can be decided by comparing their suprema and infima.
	A further important property of the greedy algorithm specialized to lattices is that it regards the input sets $A$ and $B$ above \textit{symmetrically}. 
	This is a crucial difference e.g.~to \textit{inductive logic programming}~\cite{Mugg91,NieWol97,plotkin1970inductive}, where one is typically interested in finding the smallest ideal of the {\em subsumption lattices} that contains the set of positive examples. If this smallest ideal, with supremum defined by the least general generalization of the set of positive examples, is \textit{not} disjoint with the set of negative examples, then the separation problem has no solution. This case, however, does not exclude the situation that there is a filter containing the set of positive examples that is disjoint with a ideal containing the set of negative examples.
	In addition to these properties, we also show that our modified greedy algorithm comprises an \textit{algorithmic} characterization of Kakutani closure systems over lattices. 
	This characterization provides an alternative to the algebraic one  formulated in terms of \textit{distributivity} (see, e.g., \cite{Kubis2002}). 
	
	\begin{figure}
		\label{fig:lattice_example}
\begin{tikzpicture}
	\begin{scope}[shift={(0, 0)}]
		
		\node[circle, draw=black, fill=black, inner sep=0pt, minimum size=5pt, label=right:{$\topL$}] (A) at (0, 0){};
		\node[circle, draw=black, fill=white, inner sep=0pt, minimum size=5pt] (B) at (0, -0.5){};
		\node[circle, draw=black, fill=white, inner sep=0pt, minimum size=5pt] (C) at (0, -1){};
		\node[circle, draw=black, fill=white, inner sep=0pt, minimum size=5pt] (D) at (0, -1.5){};
		\node[circle, draw=black, fill=white, inner sep=0pt, minimum size=5pt] (E) at (-1, -2){};
		\node[circle, draw=black, fill=white, inner sep=0pt, minimum size=5pt] (F) at (1, -2){};
		\node[circle, draw=black, fill=white, inner sep=0pt, minimum size=5pt] (G) at (-1, -2.5){};
		\node[circle, draw=black, fill=white, inner sep=0pt, minimum size=5pt] (H) at (0, -2.5){};
		\node[circle, draw=black, fill=white, inner sep=0pt, minimum size=5pt] (I) at (1, -2.5){};
		\node[circle, draw=black, fill=white, inner sep=0pt, minimum size=5pt] (J) at (-0.5, -3){};
		\node[circle, draw=black, fill=white, inner sep=0pt, minimum size=5pt] (K) at (0.5, -3){};
		\node[circle, draw=black, fill=black, inner sep=0pt, minimum size=5pt, label=right:{$\botL$}] (L) at (0, -3.5){};
		\draw (A)--(B)--(C)--(D)--(F)--(I)--(K)--(L)--(J)--(G)--(E)--(D);
		\draw (K)--(H)--(E);
		
		\draw[blue, fill=blue, opacity=0.2] (0,0.25) to [closed, curve through = {(-0.25,-0.5) (0,-1.25) (0.25,-0.5)}] (0,0.25);
		
		\draw[red, fill=red, opacity=0.2] (0,-1.25) to [closed, curve through = {(-1.5,-2.75) (0,-3.75) (1.5,-2.75)}] (0,-1.25);
	\end{scope}
	
	\begin{scope}[shift={(8, 0)}]
		
		\node[circle, draw=black, fill=black, inner sep=0pt, minimum size=5pt, label=right:{$\topL$}] (A) at (0, 0){};
		\node[circle, draw=black, fill=white, inner sep=0pt, minimum size=5pt] (B) at (0, -0.5){};
		\node[circle, draw=black, fill=white, inner sep=0pt, minimum size=5pt] (C) at (0, -1){};
		\node[circle, draw=black, fill=white, inner sep=0pt, minimum size=5pt] (D) at (0, -1.5){};
		\node[circle, draw=black, fill=white, inner sep=0pt, minimum size=5pt] (E) at (-1, -2){};
		\node[circle, draw=black, fill=white, inner sep=0pt, minimum size=5pt] (F) at (1, -2){};
		\node[circle, draw=black, fill=white, inner sep=0pt, minimum size=5pt] (G) at (-1, -2.5){};
		\node[circle, draw=black, fill=white, inner sep=0pt, minimum size=5pt] (H) at (0, -2.5){};
		\node[circle, draw=black, fill=white, inner sep=0pt, minimum size=5pt] (I) at (1, -2.5){};
		\node[circle, draw=black, fill=white, inner sep=0pt, minimum size=5pt] (J) at (-0.5, -3){};
		\node[circle, draw=black, fill=white, inner sep=0pt, minimum size=5pt] (K) at (0.5, -3){};
		\node[circle, draw=black, fill=black, inner sep=0pt, minimum size=5pt, label=right:{$\botL$}] (L) at (0, -3.5){};
		\draw (A)--(B)--(C)--(D)--(F)--(I)--(K)--(L)--(J)--(G)--(E)--(D);
		\draw (K)--(H)--(E);
		
		\draw[blue, fill=blue, opacity=0.2] (-0.25,0.25) to [closed, curve through = {(-0.25,-0.5) (-0.75,-1) (-0.5,-3.25) (-0.2,-2.25) (0,-1.75)}] (-0.25,0.25);
		
		\draw[red, fill=red, opacity=0.2] (1.5,-1.75) to [closed, curve through = {(0,-3) (-0.5,-3.75) (1,-3)}] (1.5,-1.75);
	\end{scope}

	\begin{scope}[shift={(4, 0)}]
		
		\node[circle, draw=black, fill=black, inner sep=0pt, minimum size=5pt, label=right:{$\topL$}] (A) at (0, 0){};
		\node[circle, draw=black, fill=white, inner sep=0pt, minimum size=5pt] (B) at (0, -0.5){};
		\node[circle, draw=black, fill=white, inner sep=0pt, minimum size=5pt] (C) at (0, -1){};
		\node[circle, draw=black, fill=white, inner sep=0pt, minimum size=5pt] (D) at (0, -1.5){};
		\node[circle, draw=black, fill=white, inner sep=0pt, minimum size=5pt] (E) at (-1, -2){};
		\node[circle, draw=black, fill=white, inner sep=0pt, minimum size=5pt] (F) at (1, -2){};
		\node[circle, draw=black, fill=white, inner sep=0pt, minimum size=5pt] (G) at (-1, -2.5){};
		\node[circle, draw=black, fill=white, inner sep=0pt, minimum size=5pt] (H) at (0, -2.5){};
		\node[circle, draw=black, fill=white, inner sep=0pt, minimum size=5pt] (I) at (1, -2.5){};
		\node[circle, draw=black, fill=white, inner sep=0pt, minimum size=5pt] (J) at (-0.5, -3){};
		\node[circle, draw=black, fill=white, inner sep=0pt, minimum size=5pt] (K) at (0.5, -3){};
		\node[circle, draw=black, fill=black, inner sep=0pt, minimum size=5pt, label=right:{$\botL$}] (L) at (0, -3.5){};
		\draw (A)--(B)--(C)--(D)--(F)--(I)--(K)--(L)--(J)--(G)--(E)--(D);
		\draw (K)--(H)--(E);
		
		\draw[blue, fill=blue, opacity=0.2] (0,0.25) to [closed, curve through = {(-0.2, -1.5) (0.95,-2.75) (1, -2.75) (1.4, -2.25) (0.2, -1)}] (0,0.25);
		
		\draw[red, fill=red, opacity=0.2] (-1.2,-1.75) to [closed, curve through = {(-1.4,-2) (0,-3.75) (0.8,-3.5) (0,-2.2)}] (-1.2,-1.75);
		
	\end{scope}
\end{tikzpicture}
		\caption{The example shows three different separations of the lattice elements $\topL$ and $\botL$ the first two ones are half-space separations while the third one is a maximal disjoint closed set separation. Note that the lattice is not distributive, i.e., our greedy algorithm does not guarantee to find a half-space separation.}
	\end{figure}
	
	%
	
	\subsubsection{Notions and Notation}
	For basic notions from \textit{lattice theory}, the reader is referred to
	\cite{Davey/Priestley/2002,Grae11}. 
	Let  $(S; \leq)$ be a partially ordered set (or poset) and $X \subseteq S$.
	The \textit{supremum} (resp. \textit{infimum}) of $X$, if it exists,
	is denoted by $\Top{X}$ (resp. $\Bot{X}$).
	For a finite lattice $(L;\leq)$, $\botL$ (resp. $\topL$) denotes $\Bot{L}$ (resp. $\Top{L}$).
	Unless otherwise stated, throughout this section by lattices we always mean \textit{finite} lattices. 	
	
	Let $(L; \leq)$ be a lattice. An element $x \in L$ is an \textit{upper} (resp. \textit{lower}) \textit{cover} of $a\in L$ if $a \leq x$ (resp. $x \leq a$) and for all $y\in L$ with $a\leq y$ (resp. $a \geq y$) it holds that $a\leq x \leq y$ (resp. $y\leq x \leq a$). 
	The set of upper (resp. lower) covers of $a$ is denoted by $\covu(a)$ (resp. $\covl(a)$).
	A lattice $L$ is called \textit{distributive} if $\meetp{a,\joinp{b,c}}=\joinp{\meetp{a,b},\meetp{a,c}}$  holds for all $a, b, c\in L$. 
	A \textit{sublattice} of $L$ is a non-empty subset of $L$ which is a lattice. 
	An {\em ideal} $I$ of $L$ is a non-empty subset of $L$ satisfying (i) $\joinp{a,b} \in I$ for all $a,b \in I$ and (ii) $a \in I$ whenever $a \in L$, $b \in I$, and $a \leq b$.
	An ideal $I \subsetneq L$ is \textit{prime} if $a\in I$ or $b\in I$ whenever $\meetp{a,b}\in I$. 
	The dual notions of ideals and prime ideals are called \textit{filters} and \textit{prime filters}, respectively. 
	One can easily check that all ideals and filters of $L$ are sublattices of $L$. Furthermore, as $|L| < \infty$ by assumption, 
	an ideal $I$ (resp. filter $F$) can be represented by $\TopI$ (resp. $\BotF$). The ideal (resp. filter) of $L$ with top (resp. bottom) element $a$ is denoted by $(a]$ (resp. $[a)$).  It follows from the definitions that the complement of a prime ideal of $L$ is a prime filter of $L$ and vice versa. 
	
	\paragraph{Closure Systems over Lattices}
	
	For finite lattices $(L; \leq)$, we will consider the usual \textit{closure} operator (see, e.g.,~\cite{vandeVel1984}), i.e., the function $\cla: 2^L \to 2^L$ defined by 
	\begin{equation}
	\label{eq:LatticeOperator}
	\cla:L'\mapsto\{x\in L~|~ \Bot{L'} \leq x\leq \Top{L'} \}
	\end{equation}
	for all $L' \subseteq L$. 
	The set $\cla(L')$ forms a \textit{closed} sublattice of $L$, where a sublattice $S$ of $L$ is {\em closed} if for all $a,b \in S$ and for all $c \in L$, $a \leq c \leq b$ implies $c \in S$. 
	Thus, $(L, \cl_\cla)$ is a \textit{closure system} formed by the family of closed sublattices of $L$ together with the empty set. This type of closure systems will be referred to as \textit{$\lambda$-closure systems}.
	In Lemmas~\ref{lemma:1} and \ref{lemma:2} below we formulate some basic properties of finite lattices and $\lambda$-closure systems.
	Though most of the claims follow from basic properties of lattices, we provide all proofs in Appendix~\ref{app:A} for the reader's convenience.

	\begin{lemma}\label{lemma:1}
		Let $(L;\leq)$ be a finite lattice and  $A, B\subseteq L$. Then the following statements are equivalent:
		\begin{enumerate}[label = (\roman*)]
			\item \label{item:2bot_top} $\NegBot\nleq \PosTop$,
			\item \label{item:2disjointness}  $[\inf B) \cap (\sup A] =\emptyset$,
			\item \label{item:2ideal_filter}    there exist an ideal $I \subseteq L$ and a filter $F \subseteq L$ with $I \cap F = \emptyset$ such that $A\subseteq I \wedge B\subseteq F$.
		\end{enumerate}
	\end{lemma}
	\begin{lemma}\label{lemma:2}
		Let $(L,\cC_\cla)$ be the $\lambda$-closure system over a finite lattice $(L;\leq)$ and $A, B\subseteq L$. Then $\cla(A)\cap\cla(B)=\emptyset$ if and only if there exist an ideal $I \subseteq L$ and a filter $F \subseteq L$ with $I \cap F = \emptyset$ such that $(A\subseteq I \wedge B\subseteq F)$ or $(B\subseteq I \wedge A\subseteq F)$.
	\end{lemma}

	\subsubsection{Maximal Closed Set Separation in Lattices}\label{sec:LatticeAlgorithm}

	Applying Thm.~\ref{thm:optimal} to $\lambda$-closure systems over a lattice $(L; \leq)$, 
	we have that Alg.~\ref{AlgoGreedyHalfSep} requires $\bigO{|L|}$ closure operator calls. 
	If the cardinality of $L$ is exponential in some parameter $n$, then the bound above becomes exponential in $n$. 
	As an example, in case of \textit{formal concept analysis}~\cite{Gan05}, the cardinality of the concept lattice can be exponential in that of the underlying sets of objects and attributes. As another example, the lattice formed by the family of \textit{closed (item)sets} of a transaction database over $n$ items can also be exponential in $n$~\cite{Boros03}.
	These and other examples motivate us to adapt Alg.~\ref{AlgoGreedyHalfSep} to lattices in a natural way, allowing for an upper bound on the number of closure operator calls in terms of the cardinalities of the \textit{upper} and \textit{lower covers} of a lattice and the maximum chain length in $L$.  
	As we show below, in case of concept lattices or (frequent) closed itemset lattices, the exponential bound above reduces to $\bigO{n^2}$. 
	
	\begin{algorithm}[t]
		\SetAlgoLined
		\caption{\sc Maximal Closed Set Separation in Lattices}
		\DontPrintSemicolon
		\LinesNumbered
		\label{AlgoLatticeSep}
		\KwIn{lattice $(L;\leq)$ with $|L| < \infty$ given by an upward and a downward refinement operator returning $\covu(a)$ and $\covl(a)$ for any $a \in L$, and $A,B \subseteq L$}
		\KwOut{supremum of a maximal ideal $I \in \cC_\cla$ and infimum of a maximal filter $F \in \cC_\cla$ separating $A$ and $B$ in $(L,\cC_\cla)$ with $\cla$ defined in (\ref{eq:LatticeOperator}) if $\cla(A)\cap \cla(B) = \emptyset$; ``{\sc No}'' o/w}
		\BlankLine
		
		\lIf{\ $\left(\TopA \ngeq \BotB\right)$}
		{$\top_I\leftarrow\TopA, \ \bot_F \leftarrow \BotB$}
		\lElseIf{$\left(\TopB \ngeq \BotA\right)$}
		{$\top_I\leftarrow\TopB, \ \bot_F \leftarrow \BotA$}
		\lElse{\Return{``{\sc No}''}}
		
		\lWhile{\label{line:upward}$\exists u \in \covu(\top_I)$ with $u\ngeq \bot_F$}
		{$\top_I \leftarrow u$}
		
		\lWhile{\label{line:downward}$\exists l \in \covl(\bot_F)$ with $l\not\leq \top_I$}
		{$\bot_F\leftarrow l$}
		
		\Return{$\top_I, \bot_F$}\;
	\end{algorithm}
	
	The algorithm solving the \MCSep-problem for finite lattices is given in Alg.~\ref{AlgoLatticeSep}. 
	In case of lattices we can assume 
	that the input lattice $(L;\leq)$ is given by an upward  $\covu$ and a downward $\covl$ refinement operator returning the sets of upper resp. lower covers for the elements of $L$.
	For any $A,B \subseteq L$, the algorithm first checks whether their closures are disjoint or not; this is decided by comparing the suprema and infima of $A$ and $B$ (cf. Lines~1--3).
	If the two closed sets are disjoint then, by Lemma~\ref{lemma:2}, $L$ has a smallest ideal $I$ and a smallest filter $F$ such that $I$ and $F$ are disjoint and either $\cla(A) \subseteq I$ and $\cla(B) \subseteq F$ or vice versa.
	The algorithm then iteratively tries to extend either $I$ into a larger ideal or $F$ into a larger filter in  a way that the extension does not violate the disjointness condition.
	In the first case, the supremum of $I$ is replaced by one of its upper covers; in the second one the infimum of $F$ by one of its lower covers. 
	Finally, the algorithm stops when the current ideal and hence, the current filter becomes prime or when any further extension makes them non-disjoint.   
	
	Alg.~\ref{AlgoLatticeSep} has some important advantageous properties over Alg.~\ref{AlgoGreedyHalfSep}. 
	In particular, while Alg.~\ref{AlgoGreedyHalfSep} considers \textit{all} uncovered elements for the extension of the current closed sets, Alg.~\ref{AlgoLatticeSep} restricts the choice of the next generator element to $\covu(\TopI) \cup \covl(\BotF)$, i.e., to a \textit{subset} of the set of elements uncovered so far. 	
	Although in the worst case this change does not improve the number of closure operator calls stated in Thm.~\ref{thm:optimal} for Alg.~\ref{AlgoGreedyHalfSep}, below we show that a logarithmic bound holds for certain closure systems over lattices. 
	Another advantageous property of Alg.~\ref{AlgoLatticeSep} is that it utilizes that the disjointness of two closed sets can be decided by comparing two elements only, i.e., the supremum of the current ideal with the infimum of the current filter. Furthermore, the closure operator can be calculated in an easy way by taking advantage of the fact that any closed sublattice of $L$ can be represented by its top and bottom elements. 
	We have the following result for Alg.~\ref{AlgoLatticeSep}:
	
	\begin{Theorem}
		\label{thm:AlgoLatticeSep}
		For any $\lambda$-closure system over a finite lattice $(L;\leq)$, Alg.~\ref{AlgoLatticeSep} solves the \MCSep~problem correctly. 
	\end{Theorem}
	\begin{proof}
		Let $(L, \cl_\cla)$ be the $\lambda$-closure system over a lattice $(L;\leq)$ and $A,B \subseteq L$.
		The correctness for the case that $\cla(A) \cap \cla(B) \neq \emptyset$ (or equivalently, $A$ and $B$ are not separable in $\cl_\cla$) is immediate from Lemmas~\ref{lemma:1} and \ref{lemma:2}.
		Applying Lemma~\ref{lemma:2} to the case that $\cla(A)\cap\cla(B)=\emptyset$, there exist disjoint ideal $I$ and filter $F$ in $\cC_\cla$ such that $A \subseteq I$ and $B \subseteq F$ or vice versa.
		For symmetry, we can assume without loss of generality that $A \subseteq I$ and $B \subseteq F$.
		Then, by Lemma~\ref{lemma:1}, the condition in Line~1 holds and thus, the algorithm terminates in Line~6 for this case.
		Consider the sequences $u_1,\ldots,u_p \in L$ and $l_1,\ldots,l_q \in L$ selected in this order in Lines~\ref{line:upward} and \ref{line:downward}, respectively.
		By construction, $A \subseteq (u_0] \subsetneq (u_1] \subsetneq \ldots \subsetneq (u_p]$ and $B \subseteq [l_0) \subsetneq [l_1) \subsetneq \ldots \subsetneq [l_q)$, where $u_0 = \TopA$ and $l_0 = \BotB$.
		Furthermore, as $u_p \ngeq l_q$ (cf. Line~\ref{line:downward}), the ideal $(u_p]$ and filter $[l_q)$ corresponding to the output $\top_I=u_p$ and $\bot_F=l_q$ are disjoint by Lemma~\ref{lemma:1}. Thus, they form a closed set separation of $A$ and $B$ in $\cC_\cla$.		
		
		We now show that $(u_p]$ and  $[l_q)$ form a \textit{maximal} closed set separation of $A$ and $B$  in $\cC_\cla$. 
		Suppose for contradiction that there exist $I',F' \in \cC_\cla$ with $I' \cap F' = \emptyset$, $(u_p] \subseteq I'$, and $[l_q) \subseteq F'$ such that at least one of the two containments is proper.
		We present the proof for $(u_p] \subsetneq I'$ only; the case of $[l_q) \subsetneq F'$ is similar. 
		Since $(u_p] \subsetneq I'$, there exists an $u \in \covu(u_p)$ with $u_p \lneq u \leq \Top{I'}$. 
		But then, by Line~\ref{line:upward} we have $u\geq l_q$, which contradicts $I'\cap F'=\emptyset$ by Lemma~\ref{lemma:2}, as $\Top{I'}\geq u \geq l_q \geq \Bot{F'}$.
		%
	\end{proof}
	
	One can easily check that the number of evaluations of the relation ``$\leq$'' in Lines~\ref{line:upward} and \ref{line:downward} is $\bigO{H_L C_L}$, where $H_L$ is the maximum chain length in $L$ and $C_L = \max\limits_{x \in L} \max \{|\covu(x)|,|\covl(x)|\}$.
	Proposition~\ref{thm:setlattices} below utilizes this property for the special case that the underlying lattice is a family of subsets of some finite ground set.
	\begin{Proposition}
		\label{thm:setlattices}
		Let $(L,\cC_\cla)$ be a $\lambda$-closure system over a lattice $(L;\subseteq)$ with $L \subseteq 2^E$ for some ground set $E$ of cardinality $n$.
		Then Alg.~\ref{AlgoLatticeSep} solves the \MCSep~problem for the $\lambda$-closure system over $(L;\subseteq)$ with $\bigO{n^2}$ evaluations of the subset relation.
	\end{Proposition}
	\begin{proof}
		It follows directly from the remark above by $H_L = \bigO{n}$ and $H_C = \bigO{n}$.
	\end{proof}
	Since $L\subseteq 2^E$ for some ground set $E$ with $|E|=n$, $|L|$ can be exponential in $n$. As an application of Proposition \ref{thm:setlattices} to concept lattices and closed (frequent) itemsets, we have that maximal closed separations can be found in time polynomial in the size of the underlying ground sets for these types of closure systems.
	In Appendix~\ref{app:B} we present an illustrative example of an application of Alg.~\ref{AlgoLatticeSep} to \textit{learning} in first-order logic.

	\subsubsection{Kakutani Closure Systems over Lattices}
	\label{sec:LatticesKakutani}
	In this section we consider \textit{Kakutani} $\lambda$-closure systems over finite lattices.
	This kind of closure systems have a well-known \textit{algebraic} characterization in terms of distributivity (see, e.g., \cite{Kubis2002,vandeVel1984}).
	As an orthogonal result, in Thm.~\ref{thm:kakutanilattice} below we show that Alg.~\ref{AlgoLatticeSep} provides an \textit{algorithmic} characterization of Kakutani $\lambda$-closure systems over lattices.

	\begin{Theorem}
		\label{thm:kakutanilattice}
		Let $(L, \cl_\cla)$ be the $\lambda$-closure system over a finite lattice $(L;\leq)$.
		Then $(L, \cl_\cla)$ is Kakutani if and only if  for all  non-empty $A,B \subseteq L$ with $\cla(A)\cap \cla(B)=\emptyset$, the ideal $(\top_I]$ and filter $[\bot_F)$ defined by the output $\top_I$ and $\bot_F$ of Alg.~\ref{AlgoLatticeSep} form a partitioning of $L$. 
	\end{Theorem}
	\begin{proof}
		The sufficiency is immediate by Thm.~\ref{thm:AlgoLatticeSep}.
		For the necessity, let $(L, \cl_\cla)$ be a Kakutani closure system and $A,B \subseteq L$ with $\cla(A)\cap \cla(B)=\emptyset$. 
		Let $u_1,\ldots,u_p$ and $l_1,\ldots,l_q$ be the maximal sequences considered in the proof of Thm.~\ref{thm:AlgoLatticeSep} for the case of $\cla(A)\cap \cla(B)=\emptyset$.
		For their last elements we have that $(u_p],[l_q) \in \cC_\cla$ and $(u_p]\cap [l_q) = \emptyset$.
		Since $\cl_\cla$ is Kakutani, there is a proper partitioning $H,H^c \in \cC_\cla$ of $L$ such that $(u_p] \subseteq H$ and $[l_q) \subseteq H^c$. 
		Thus, $\botL \in H$ and $\topL \in H^c$, implying that $H$ is a prime ideal and $H^c$ is its complement prime filter.
		Suppose for contradiction that one of the two containments above, say the first one, is proper (i.e., $(u_p] \subsetneq H$).
		But then, using similar arguments as in the proof of Thm.~\ref{thm:AlgoLatticeSep}, there exists an element $u \in \covu(u_p)$ such that $u$ will be selected after $u_q$ in Line~\ref{line:upward}. This contradicts that $u_q$ is the last element selected in Line~\ref{line:upward}. 
	\end{proof}
	\begin{Corollary}\label{cor:AlgorithmicCharacterization}
		Let $(L, \cl_\cla)$ be the $\lambda$-closure system over a finite lattice $(L;\leq)$.
		Then $(L, \cl_\cla)$ is distributive if and only if  for all  non-empty $A,B \subseteq L$ with $\cla(A)\cap \cla(B)=\emptyset$, the ideal $(\top_I]$ and filter $[\bot_F)$ defined by the output $\top_I$ and $\bot_F$ of Alg.~\ref{AlgoLatticeSep} form a partitioning of $L$. 
	\end{Corollary}
	\begin{proof}
		Immediate from Thm.~\ref{thm:kakutanilattice} and the characterization of $\lambda$-closure systems over finite lattices in terms of distributivity~\cite{Kubis2002,vandeVel1984}.
	\end{proof}

	\section{Concluding Remarks}\label{sec:summary}
	In \cite{Seiffarth20} we have ``enriched'' the data by  defining an abstract ``distance'' function between elements and sets, using the notion of \textit{monotone linkage functions}~\cite{Mullat1975}.\footnote{Monotone linkage functions provide an abstract measure of \textit{connectedness} between objects. They are used, among others, in the context of clustering in \textit{convex geometries}~\cite{Kempner1997,Kempner2003}.
	}
	Motivated by the concept of support vector machines~\cite{Vapnik/98}, in \cite{Seiffarth20} we study the problem of half-space and maximal closed set separation with \textit{maximum margin} in this kind of \textit{abstract} closure systems.
	
	The results of Section~\ref{sec:LatticesKakutani} show that Alg.~\ref{AlgoGreedyHalfSep} can be improved in terms of the number of closure operator calls by utilizing structural properties of lattices.	
	It would be interesting to study Alg.~\ref{AlgoGreedyHalfSep} for other domains, in particular, for special classes of graphs and relational structures from this point of view.
	Throughout this paper we considered \textit{binary} separation problems only. Clearly, they can naturally be extended to \textit{multi-class} separation problems (i.e., finding a $k$-partitioning of the ground set or $k$ maximal closed sets that are pairwise disjoint, for some $k \geq 2$ integer).
	While the generalization of our results concerning maximal closed set separation is straightforward, it is less obvious for the $k$-partitioning problem, which could also be called $k$-Kakutani (i.e., $2$-Kakutani is the special case of the half-space separation problem).
	We note that for the special case of graphs, this problem has already been studied by \cite{ARTIGAS20111968}.

	In the general problem settings considered in this work we assumed that the closure operator is given by an oracle, which returns the closure of a set \textit{extensionally}. In case of closure systems over lattices, closed sets (e.g., ideals and filters) can, however, be represented \textit{intensionally} (i.e., by their suprema and infima).
	As another example, for closure systems over trees we have that any half-space has a succinct intensional representation e.g. by a single node together with the edge connecting it to the complement half-space.
	These and other examples motivate the study of structural properties of closure systems allowing for some compact \textit{intensional} representation of abstract half-spaces and closed sets.
	%
	
	A further problem is to study algorithms solving the HSS and \MCSep~problems for closure systems, for which an upper bound on the VC-dimension is known in advance. The relevance of the VC-dimension in this context is that for any closed set $C \in \cl_\cc$ of a closure system $(E,\cl_\cc)$, there exists a set $G \subseteq E$ with $|G| \leq d$ such that $\cc(G) = C$, where $d$ is the VC-dimension of $\cl_\cc$ (see, e.g., \cite{Horvath/Turan/01}). It is an open question whether the lower bound on the number of closure operator calls can be characterized in terms of the VC-dimension of the underlying closure system. 
	Finally, regarding Kakutani closure systems, it is an interesting research direction to study the relaxed notion of {\em almost} Kakutani closure systems, i.e., in which the combined size of the output closed sets are close to the cardinality of the ground set.

	\subsection*{Acknowledgements}
	We thank Victor Chepoi for his valuable comments on an early version of this paper.
	This work was partly supported by the Ministry of
	Education and Research of Germany (BMBF) under project ML2R (grant number 01/S18038C) and by the Deutsche Forschungsgemeinschaft (DFG,
	German Research Foundation) under Germany's Excellence Strategy - EXC
	2070 - 390732324.


	\appendix
	\section{Proofs}
	\label{app:A}
	
	\begin{sloppypar}
	\begin{proof}[Proof of Theorem~\ref{pr:HSP}]  	
	Let $G=(V,E)$ be an instance of the \textsc{Convex 2-Partitioning} problem and $\cg$ the closure operator corresponding to the closure system defined in (\ref{eq:clgamma}).
	It holds that $G$ has a proper convex $2$-partitioning if and only if there are $u,v \in V$ with $u \neq v$ such that $\cg(\{u\})$  and  $\cg(\{v\})$ are half-space separable in $(V,\cl_\cg)$.
	Indeed, if $G$ has a proper convex $2$-partitioning then there exist $u,v \in V$ belonging to different convex partitions. Since the two convex partitions are (complementary) half-spaces in $(V,\cl_\cg)$, $\{u\}$ and $\{v\}$ are half-space separable in $(V,\cl_\cg)$.  
	Conversely, if there are $u, v\in V$ such that $\{u\}$ and $\{v\}$ are half-space separable in $(V,\cl_\cg)$, then the corresponding half-spaces form a proper convex 2-partitioning of $G$.
	Putting together, the \textsc{Convex 2-Partition} problem can be decided by solving the HSS~problem for the input $(V,\cl_\cg)$, $A= \{u\}$, and $B = \{v\}$ for all $u,v \in V$. This completes the proof, as the number of vertex pairs is quadratic in the size of $G$.
	\end{proof}
	\end{sloppypar}

	\begin{proof}[Proof of Lemma~\ref{lemma:1}]
		For \ref{item:2bot_top}$\implies$\ref{item:2disjointness}, suppose for contradiction that $[\inf B) \cap (\sup A]\neq\emptyset$. Then there is an $x \in L$ with $\NegBot\leq x$ and $x \leq \PosTop$, contradicting (i). The proof of \ref{item:2disjointness} $\implies$ \ref{item:2ideal_filter} follows directly from the fact that $[\inf B)$ is a filter and $(\sup A]$ an ideal. Regarding \ref{item:2ideal_filter}$\implies$\ref{item:2bot_top}, it must be the case that $\BotB \nleq \TopA$, as otherwise $\BotF \leq \BotB \leq \TopA \leq \TopI$, contradicting the disjointness of $I$ and $F$.
	\end{proof}  

	\begin{proof}[Proof of Lemma~\ref{lemma:2}]
		The proof of the ``\textit{if}'' direction is immediate by $I,F \in \cC_\cla$.
		Regarding the other direction, we first claim that $\cla(A)\cap\cla(B)=\emptyset$ implies $\NegBot\nleq \PosTop$ or  $\PosBot\nleq \NegTop$. Suppose for contradiction that $\NegBot\leq \PosTop$ and $\PosBot\leq \NegTop$. Then 
		$\NegBot\leq \joinp{\PosBot,\NegBot} \leq \NegTop$ and $\PosBot\leq \joinp{\PosBot,\NegBot} \leq \PosTop$, implying $\joinp{\PosBot,\NegBot}\in \cla(A) \cap \cla(B)$, which contradicts $\cla(A)\cap\cla(B)=\emptyset$. The claim then follows from Lemma~\ref{lemma:1} by the symmetry of $A$ and $B$. 
	\end{proof}

\section{Example}
\label{app:B}
	\begin{figure}
	\rotatebox{90}{
	\begin{tikzpicture}[scale=0.9, every node/.style={scale=0.9}]
	\newcommand\YStep{1.3}
	\newcommand\XStep{1.75}
	\newcommand\YGap{0}
	
	\draw[fill=red!10] (5.2*\XStep, -3.2*\YStep) to [closed, curve through = {(0*\XStep, -2.5*\YStep)(-0.5*\XStep, -2*\YStep) (-0.5*\XStep, 0*\YStep) (-0.5*\XStep, 2.5*\YStep) (0.5*\XStep, 2.5*\YStep) (1*\XStep, 1*\YStep) (2.5*\XStep, 0.2*\YStep) (2.5*\XStep, -0.2*\YStep) (5*\XStep, -1.5*\YStep) (5*\XStep, -2*\YStep)}] (5.2*\XStep, -3.2*\YStep);
	
	\draw[fill=blue!10] (3.5*\XStep, 4*\YStep) to [closed, curve through = {(3.5*\XStep, 3.5*\YStep) (3.5*\XStep, 2*\YStep) (5.5*\XStep, 0*\YStep) (6*\XStep, -0.5*\YStep) (7*\XStep, -2*\YStep) (10*\XStep, -2*\YStep) (10*\XStep, 0*\YStep) (10*\XStep, 2*\YStep) (7*\XStep, 3*\YStep)}] (3.5*\XStep, 4*\YStep);
	
	\tiny
	\node (top) at (4.5*\XStep,4*\YStep) {$P(v, w, x, y, z)$};
	
	\node[draw=red!50, rounded corners=5pt, dashed, thick] (x1) at (0*\XStep,2*\YStep) {$P(v, v, x, y, z)$};
	\node (x2) at (1*\XStep,2*\YStep+\YGap) {$P(v, w, v, y, z)$};
	\node (x3) at (2*\XStep,2*\YStep-\YGap) {$P(v, w, x, v, z)$};
	\node (x4) at (3*\XStep,2*\YStep+\YGap) {{$P(v, w, x, y, v)$}};
	\node (x5) at (4*\XStep,2*\YStep-\YGap) {$P(v, w, w, y, z)$};
	\node (x6) at (5*\XStep,2*\YStep+\YGap) {{$P(v, w, x, w, z)$}};
	\node (x7) at (6*\XStep,2*\YStep-\YGap) {$P(v, w, x, y, w)$};
	\node[fill=blue!50, rounded corners=5pt] (x8) at (7*\XStep,2*\YStep+\YGap) {$P(v, w, x, x, z)$};
	\node (x9) at (8*\XStep,2*\YStep-\YGap) {$P(v, w, x, y, x)$};
	\node[fill=blue!50, rounded corners=5pt] (x10) at (9*\XStep,2*\YStep+\YGap) {$P(v, w, x, y, y)$};
	
	\node[draw=red!50, rounded corners=5pt, thick] (y1) at (0*\XStep,0*\YStep-\YGap) {$P(v, v, v, y, z)$};
	\node (y2) at (1*\XStep,0*\YStep+\YGap) {$P(v, v, x, v, z)$};
	\node (y3) at (2*\XStep,0*\YStep-\YGap) {$P(v, v, x, v, z)$};
	\node (y4) at (3*\XStep,0*\YStep+\YGap) {$P(v, w, v, v, z)$};
	\node (y5) at (4*\XStep,0*\YStep-\YGap) {$P(v, w, v, y, v)$};
	\node (y6) at (5*\XStep,0*\YStep+\YGap) {$P(v, w, x, v, v)$};
	\node (y7) at (6*\XStep,0*\YStep-\YGap) {$P(v, w, w, w, z)$};
	\node (y8) at (7*\XStep,0*\YStep+\YGap) {$P(v, w, w, y, w)$};
	\node (y9) at (8*\XStep,0*\YStep-\YGap) {$P(v, w, x, w, w)$};
	\node[draw=blue!50, rounded corners=5pt, thick] (y10) at (9*\XStep,0*\YStep-\YGap) {$P(v, w, x, x, x)$};
	
	\node[fill=red!50, rounded corners=5pt] (z1) at (0.5*\XStep,-2*\YStep) {$P(v, v, v, v, z)$};
	\node[fill=red!50, rounded corners=5pt] (z2) at (2.5*\XStep,-2*\YStep) {$P(v, v, v, y, v)$};
	\node (z3) at (4.5*\XStep,-2*\YStep) {$P(v, v, x, v, v)$};
	\node (z4) at (6.5*\XStep,-2*\YStep) {$P(v, w, v, v, v)$};
	\node[draw=blue!50, rounded corners=5pt, dashed, thick] (z5) at (8.5*\XStep,-2*\YStep) {$P(v, w, w, w, w)$};
	
	\node (bot) at (4.5*\XStep,-4*\YStep) {$P(v,v,v,v,v)$};
	
	\foreach \x in {1,2,...,10} {
	\draw (top) -- (x\x);
	}  
	
	\foreach \x in {1,2,...,5} {
	\draw (z\x) -- (bot);
	}  
	
	\foreach \x in {1, 2, 3} {
	\draw (x1) -- (y\x);
	}  
	
	\foreach \x in {1, 4, 5} {
	\draw (x2) -- (y\x);
	} 
	\foreach \x in {2, 4, 6} {
	\draw (x3) -- (y\x);
	} 
	\foreach \x in {3, 5, 6} {
	\draw (x4) -- (y\x);
	} 
	\foreach \x in {1, 7, 8} {
	\draw (x5) -- (y\x);
	} 
	\foreach \x in {2, 7, 9} {
	\draw (x6) -- (y\x);
	} 
	\foreach \x in {3, 8, 9} {
	\draw (x7) -- (y\x);
	} 
	\foreach \x in {4, 7, 10} {
	\draw (x8) -- (y\x);
	} 
	
	\foreach \x in {5, 8, 10} {
	\draw (x9) -- (y\x);
	} 
	\foreach \x in {6, 9, 10} {
	\draw (x10) -- (y\x);
	} 
	\foreach \x in {1, 2, 4, 7} {
	\draw (z1) -- (y\x);
	} 
	\foreach \x in {1, 3, 5, 8} {
	\draw (z2) -- (y\x);
	} 
	\foreach \x in {2, 3, 6, 9} {
	\draw (z3) -- (y\x);
	} 
	\foreach \x in {4, 5, 6, 10} {
	\draw (z4) -- (y\x);
	} 
	\foreach \x in {7, 8, 9, 10} {
	\draw (z5) -- (y\x);´
	}


	\end{tikzpicture}
}
	\caption{\label{fig:atomlattice}
		Maximal closed set separation of set $A$ (marked by blue) and set $B$ (marked by red) in a subsumption lattice. The elements with bold border are the supremum of $B$ resp. the infimum of $A$. The elements with dashed border are the output elements $\bot_F$  resp.~$\top_I$.}

	\end{figure}

%
%
%
%
We present an illustrative example of the application of Alg.~\ref{AlgoLatticeSep} to lattices.
It is concerned with finding consistent hypotheses in \textit{inductive logic programming} (see, e.g., \cite{NieWol97}).
For simplicity, the example below is restricted to a very simple first-order vocabulary by noting that the same idea holds for any finite sublattice of a \textit{subsumption lattice} (cf.~\cite{NieWol97} for the definition and some formal properties of  subsumption lattices).
More precisely, in the example below we assume that the vocabulary consists of a single predicate symbol $P$ of arity $n$ and a set $V$ of variables.
An atom $P(t_1,\ldots,t_n)$ with $t_1,\ldots,t_n\in V$ \textit{generalizes}  $P(t'_1,\ldots,t'_n)$, denoted $P(t_1,\ldots,t_n) \geq P(t'_1,\ldots,t'_n)$, if there exists a function $\sigma: V \to V$ such that $P(\sigma(t_1),\ldots,\sigma(t_n)) = P(t'_1,\ldots,t'_n)$. 
Two atoms $A_1,A_2$ are \textit{equivalent} if $A_1 \leq A_2$ and $A_2 \leq A_1$.
Let $L$ be a maximal set of $P$-atoms, each of the above form, that contains no two equivalent atoms.
Clearly, each element of $L$ can be represented by any atom from its equivalence class. 
It holds that $(L;\leq)$ is a finite lattice. Furthermore, the top (resp. bottom) element of $L$ is a $P$-atom such that all variables in it are pairwise different (resp. are the same).
Using the above notions and notation, we are ready to formulate our example.

	Consider the lattice $(L;\leq)$ defined above for $P$ with $n=5$ and $V = \{v,w,x,y,z\}$. For simplicity, we use a ``caconical'' atom for representing the class of equivalent atoms.
	We can assume w.l.o.g. that $\botL = P(v,v,v,v,v)$ and $\topL
	= P(v,w,x,y,z)$. 
	Let 
	\begin{eqnarray*}
		A &=& \{P(v,w,x,x,z), P(v,w,x,y,y)\}  \\
		B &=& \{P(v,v,v,v,z), P(v,v,v,y,v)\}
	\end{eqnarray*} 
	denote the sets of positive and negative examples, respectively.
	In the most common problem setting in inductive logic programming~\cite{NieWol97}, one is interested in finding a $P$-atom $g \in L$ such that $g$ generalizes all elements of $A$ and none of the elements in $B$, if such a $g$ exists. Clearly, such a $g$ exists if and only if $\TopA = P(v, w, x, y, z)$ does not generalize any of the $P$-atoms in $B$. Since this is not the case for our example, the consistent hypothesis finding problem has no solution.
		\begin{sloppypar}
		If, however, we only require $A$ and $B$ to be separable in $(L,\cC_\cla)$, then Alg.~\ref{AlgoLatticeSep} returns a  solution.
		Indeed, while 
		$$
			\Top{A} 
				= P(v, w, x, y, z) 
				\geq P(v,v,v,v,v) =
				\Bot{B} \enspace ,
		$$
		for $\Top{B}$ and $\Bot{A}$ we have 
		$$
			\Top{B} 
				= P(v,v,v,y,z) 
				\ngeq P(v,w,x,x,x) 
				= \Bot{A} \enspace ,
		$$ 
		implying $\top_I=P(v,v,v,y,z)$ and $\bot_F =P(v,w,x,x,x,)$ for Lines~1 and 2 of Alg.~\ref{AlgoLatticeSep}, which, in turn, are extended in Lines~4 and 5 into 
		$\top_I = P(v,v,x,y,z)$ and $\bot_F = P(v,w,w,w,w)$ (see, also, Fig.~\ref{fig:atomlattice}). 
		For the corresponding ideal $(P(v,v,x,y,z)]$ and filter $[P(v,w,w,w,w))$ returned by Alg.~\ref{AlgoLatticeSep} we have that they are disjoint, contain $B$ and $A$, respectively, and are \textit{maximal} in $(L,\cC_\cla)$ with respect to these properties. In other words, the output of the algorithm separates $A$ and $B$ by the sets of $P$-atoms that are generalizations of $P(v,w,w,w,w)$ and are generalized by $P(v,v,x,y,z)$, respectively. 
		This example also shows that our approach is able to produce an output for such cases where traditional approaches based on Plotkin's least general generalization~\cite{plotkin1970inductive} have no solution. The reason is that Alg.~\ref{AlgoLatticeSep} treats the input two sets symmetrically, in contrast to all such approaches.
		\end{sloppypar} 

\end{document}